\documentclass[twoside,leqno,twocolumn]{article}

\usepackage[letterpaper]{geometry}

\usepackage{siamproceedings}

\usepackage[T1]{fontenc}
\usepackage{amsfonts}
\usepackage{graphicx}
\usepackage{epstopdf}
\usepackage{enumitem}
\usepackage{algorithmic}
\ifpdf
  \DeclareGraphicsExtensions{.eps,.pdf,.png,.jpg}
\else
  \DeclareGraphicsExtensions{.eps}
\fi

\newsiamremark{remark}{Remark}
\newsiamremark{hypothesis}{Hypothesis}
\crefname{hypothesis}{Hypothesis}{Hypotheses}
\newsiamthm{claim}{Claim}

\usepackage{amsopn}

\usepackage[numbers,sort&compress]{natbib}
\usepackage{import}
\usepackage{amsmath,amsfonts,bm}

\def\eqref#1{equation~\ref{#1}}

\def\1{\bm{1}}

\def\rvg{{\mathbf{g}}}

\def\rvx{{\mathbf{x}}}

\DeclareMathAlphabet{\mathsfit}{\encodingdefault}{\sfdefault}{m}{sl}
\SetMathAlphabet{\mathsfit}{bold}{\encodingdefault}{\sfdefault}{bx}{n}

\DeclareMathOperator*{\argmax}{arg\,max}

\newtheorem{assumption}[theorem]{Assumption}
\usepackage{booktabs}
\usepackage{arydshln}
\usepackage{subfigure}
\usepackage{multirow}
\usepackage{thm-restate}
\usepackage[flushleft]{threeparttable}
\usepackage[many]{tcolorbox}

\begin{document}

\newcommand\relatedversion{}

\title{Towards Reliability of Parameter-Free Optimization}
\author{
Yijiang Pang\thanks{Michigan State University, East Lansing, USA}
\qquad  
Shuyang Yu\footnotemark[1] 
\qquad 
Bao Hoang\footnotemark[1] 
\qquad 
Jiayu Zhou\thanks{University of Michigan, Ann Arbor, USA. Corresponding author: \texttt{jiayuz@umich.edu}}
}

\date{}

\maketitle


\fancyfoot[R]{\scriptsize{Copyright \textcopyright\ 2026 by SIAM\\
Unauthorized reproduction of this article is prohibited}}





\begin{abstract} Hyperparameter tuning, particularly the selection of an appropriate learning rate in adaptive gradient training methods, remains a challenge. To address this challenge, we propose a novel parameter-free optimizer, \textsc{AdamG} (Adam with the Golden step size), designed to automatically adapt to diverse optimization problems without task-specific learning-rate tuning. The core technique underlying \textsc{AdamG} is our golden step size derived for the AdaGrad-Norm algorithm, which is expected to help AdaGrad-Norm preserve tuning-free convergence and approximate the optimal step size in expectation across diverse optimization scenarios. To better evaluate tuning-free performance, we propose a novel evaluation criterion, \textit{reliability}, to comprehensively assess the efficacy of parameter-free optimizers in addition to classical performance criteria. Empirical results demonstrate that \textsc{AdamG} outperforms other parameter-free baselines and consistently performs on par with Adam using a manually tuned learning rate across various optimization tasks.
\end{abstract}

\section{Introduction}
\label{sec_introduction}

Optimization serves as a foundational technique underpinning modern deep learning, with applications in various domains such as computer vision, AI for science, and natural language processing~\citep{voulodimos2018deep, redmon2016you, paul2021artificial,devlin2018bert, radford2019language}.
Common optimization approaches include momentum-based methods~\citep{sutskever2013importance} and Adam~\citep{kingma2014adam}.
Among them, 
adaptive gradient methods play an important role~\citep{duchi2011adaptive, kingma2014adam, liu2023sophia} due to their attractive performance across diverse problem structures, encompassing deep model architectures, data characteristics, and training hyperparameters.
Hyperparameter tuning associated with those optimization algorithms has a significant impact on practical performance~\citep{wilson2017marginal}. 
%
In particular, learning-rate (LR) tuning is important because the appropriate LR for popular adaptive-gradient methods depends on unknown problem properties, such as smoothness, gradient-estimation error, and the initial optimality gap. This dependence necessitates careful LR selection. Manual tuning is common, but it requires substantial computational resources and can be prohibitive in large-scale machine-learning tasks.

%
%
Recently, there has been a growing interest in parameter-free training methods due to their practical training efficiency and satisfactory performance.
These methods are designed to eliminate the need for manual parameter tuning, achieving performance levels close to the best manually tuned training methods.
Pioneering works in the realm of parameter-free training, incorporating mechanisms like subroutine-based and bisection-based mechanisms~\citep{nesterov2015universal, carmon2022making}, are prohibitively expensive in the context of large-scale deep learning problems.
This study directs its focus toward identifying parameter-free training methods that maintain comparable training costs to the most standard training algorithms, such as Adam~\citep{kingma2014adam}, for deep learning problems.

Current parameter-free training methods commonly incorporate the initial optimality gap into the step size or draw inspiration from the recently proposed DoG method DoG~\citep{ivgi2023dog}, which combines the classical results from AdaGrad-Norm step size and standard Gradient Descent (GD) step size~\citep{duchi2011adaptive, ward2020adagrad, khaled2023dowg}.
Existing approaches based on classical evaluation criteria demonstrate advantages in specific scenarios~\citep{ivgi2023dog, defazio2023learning, khaled2023dowg, mishchenko2023prodigy}.
However, we observed that the state-of-the-art parameter-free optimizers exhibit unstable performance when confronted with diverse optimization problems, i.e., prior methods sometimes perform much worse than the best manually tuned optimizer for some optimization tasks. This observation is supported by the experimental results in Section~\ref{subsec_comp}, 
where a default choice like Adam(1e-3) with a cosine-decay learning-rate schedule outperformed existing parameter-free optimizers in certain optimization scenarios. This prompts the following question:
\textit{
How consistently can a parameter-free optimizer achieve performance ``close'' to that of the best manually tuned optimizer across diverse optimization scenarios?
}

%
%

To tackle this problem, we first explore how to systematically evaluate the effectiveness of parameter-free optimizers.
Existing approaches mainly use the classical evaluation criteria, including \textit{convergence speed} and \textit{solution quality}~\citep{kingma2014adam, liu2023sophia} for optimizers. 
However, in the context of parameter-free optimizers, limiting evaluation to these two aspects has hindered researchers and engineers from confidently applying these optimizers to more complicated real-world tasks.
Given that a parameter-free optimizer is inherently expected to generalize to unseen optimization problems, it is critical to evaluate whether it performs consistently across a spectrum of optimization problems. 
%
%
To this end, we introduce an additional novel evaluation criterion, \textit{reliability}, for parameter-free training methods. This criterion evaluates whether a parameter-free optimizer consistently achieves performance close to that of the best manually tuned optimizer across various optimization problems. 


In this paper, we design a novel algorithm
that leverages AdaGrad-Norm's tuning-free convergence~\citep{duchi2011adaptive, ward2020adagrad, mcmahan2010adaptive, wang2023convergence}. 
Specifically, we formally define a golden step size for AdaGrad-Norm, drawing insights to preserve the ability of tuning-free convergence and approximate the optimal step size in expectation across various optimization problems. 
Subsequently, we derive the golden step size, which is independent of problem-specific properties, and integrate it into AdaGrad-Norm, resulting in our first parameter-free optimizer (Algorithm~\ref{alg_god}).
By deeply integrating the derived golden step size into Adam, we further introduce an Adam-like parameter-free method named \textsc{AdamG} (Algorithm~\ref{alg_adamg}).
Compared to existing parameter-free optimization methods, our proposed \textsc{AdamG} consistently outperforms all the baselines across various optimization tasks and achieves performance that closely aligns with the performance of the best manually tuned Adam configuration.

We highlight the following contributions of the paper:
\begin{itemize}[noitemsep, leftmargin=1cm, topsep=0pt]
    \item We introduce a novel evaluation criterion, namely \emph{reliability}, for assessing parameter-free training methods. Empirical results show that this criterion reasonably validates the adaptability of parameter-free optimizers to diverse optimization problems.
    \item Based on our analysis of the classical AdaGrad-Norm algorithm, we propose the golden step size for AdaGrad-Norm, which is expected to preserve the ability of tuning-free convergence and approximate the optimal step size in expectation across various optimization problems, resulting in a parameter-free variant of AdaGrad-Norm. Furthermore, we extend this concept to devise an Adam-like parameter-free method named \textsc{AdamG}.
    \item Extensive experiments conducted on deep learning tasks reveal that \textsc{AdamG} exhibits stable performance across a spectrum of optimization problems. Moreover, it closely aligns with the best performance achieved by manually tuning Adam.
\end{itemize}

\section{Related Work}
Adaptive gradient methods have emerged as widely used choices for training deep-learning models~\citep{kingma2014adam, balles2018dissecting, zhuang2020adabelief}. Concurrently, adaptive parameter-free approaches have gained popularity in the optimization landscape~\citep{ivgi2023dog, defazio2023learning, khaled2023dowg, mishchenko2023prodigy}.

Search is a natural way to achieve parameter-free optimization. As mentioned earlier, several works employ search mechanisms~\citep{nesterov2015universal, feurer2019hyperparameter, carmon2022making}.
Pioneering efforts without search mechanisms usually estimate the problem properties and are more concerned with guarantees for convex optimization problems.
For instance, the Polyak step size schedule incorporates $f(\rvx_{k}) - f^{\star}$ into gradient descent for convex optimization problems~\citep{polyak1987introduction}.
The subsequent adaptations of this approach demonstrate fair performance in handling nonconvex problems~\citep{loizou2021stochastic, malitsky2019adaptive, latafat2023adaptive}.
Contrary to problem properties estimation, approaches adapted from online learning with theoretical guarantees, such as coin betting schemes, have been applied in deep learning optimization problems~\citep{orabona2016coin, orabona2017training, chen2022better}.
A more recent trend involves the utilization of the initial optimality gap and the sum of squared gradient norms along training trajectory (over $K$ steps), $\frac{\max_{i\leq K}||\rvx_{i} - \rvx_{0}||\xrightarrow[]{\text{Approx.}}||\rvx_{0} - \rvx^{\star}||}{\sqrt{\sum_{i=1}^{K}||\rvg_{i}||^{2}}}$ where $\rvx$ denotes parameters and $\rvg$ denotes stochastic gradient, as a means to adapt to unknown properties associated with gradient norms, smoothness, and strong convexity~\citep{ivgi2023dog}. 
The distance term $||\rvx_{i} - \rvx_{0}||$ primarily draws from classical results using gradient descent for the convex problems, while the gradient norm $1/{\sqrt{\sum_{i=1}^{K}||\rvg_{i}||^{2}}}$ is inspired by AdaGrad~\citep{bubeck2015convex, nesterov2018lectures, duchi2011adaptive}. However, the combination itself lacks convincing theoretical guarantees over nonconvex problems.
Several works following this line of thought propose variants of the distance measure or integrate these techniques with Adam
~\citep{defazio2023learning, khaled2023dowg, mishchenko2023prodigy}.
\section{Method}
\label{sec_analysis}
%
In Section~\ref{subsec_gold_ada}, we start by analyzing and discussing the selection of LR that preserves AdaGrad-Norm's tuning-free convergence~\citep{duchi2011adaptive, mcmahan2010adaptive, wang2023convergence, faw2023beyond}.
Then, incorporating the ability with the classical result of the descent lemma of smooth function, coupled with our idea about optimizing the solution across various optimization problems, we formulate and derive the corresponding solution for the \textit{golden step size} of
AdaGrad-Norm. This golden step size is expected to help AdaGrad-Norm \textit{converge without tuning and approximate the optimal step size over various settings}.
Finally, we discuss the scale-free property of the golden step size.
These insights serve as the foundational principles for the development of our parameter-free optimizers, detailed in Section~\ref{subsec_alg1} and Section~\ref{subsec_alg2}.

\subsection{Preliminary}

We consider a differentiable, possibly non-convex function $f(\cdot):\mathbb{R}^{d}\rightarrow  \mathbb{R}$ with the standard Euclidean norm $||\cdot||$. 
We follow the standard assumptions on the objective function and stochastic gradients used in~\cite{wang2023convergence}.

\begin{assumption}[$L$-smooth condition]
\label{ass_lsmooth}
    We assume that for any model parameters $\rvx_{1},\rvx_{2}$, $f$ is differentiable and has an $L$-Lipschitz gradient such that $
    ||\nabla f(\rvx_{1}) - \nabla f(\rvx_{2})|| \leq 
    L||\rvx_{1} - \rvx_{2}||.
    $
\end{assumption}
\begin{assumption}[Affine noise variance]
\label{ass_l0l1}
    We assume that there exist positive constants $D_{0}$ and $D_{1}$ such that $\mathbb{E}_{\mathcal{F}_{k}}[||\rvg_{k}||^{2}] \leq D_{0} + D_{1}||\nabla f(\rvx_{k})||^{2},\  \forall k \geq 1$. $\mathcal{F}_{k} = \sigma(g_{k-1}, \cdots, g_{1})$ is the standard stochastic operator and stands for the sigma field of historical gradients up to $k-1$. 
\end{assumption}

\subsection{Golden Step Size for AdaGrad-Norm}
\label{subsec_gold_ada}
AdaGrad-Norm converges when optimizing non-convex objectives under affine noise variance and bounded smoothness assumptions~\citep{wang2023convergence, faw2023beyond}. 
Additionally, It enjoys tuning-free convergence, wherein differences in initial learning rates solely impact practical convergence speed rather than the final convergence. 
This attribute is considered a primary advantage inherited by subsequent variants.
We initiate our analysis with the following corollary, which serves as the foundation of analyzing the preservation of the tuning-free convergence ability of AdaGrad-Norm (cf. Algorithm~\ref{alg_god}).
\begin{corollary}[A variant of Thm.~2 in \cite{wang2023convergence}]
\label{corollary_convergence}
Given Assumptions~\ref{ass_lsmooth} and~\ref{ass_l0l1}, for AdaGrad-Norm with any learning rate $\eta >0$, the following holds in expectation:
$
\min_{k\in [K]} ||\nabla f(\rvx_{k})||^{2} \leq \frac{1}{\mathcal{O}(\sqrt{v_{K}})}\bigg(\frac{4}{\eta}(f(\rvx_{1}) - f^{\star}) + 2D_{1}\xi(0)
 + \left(2(L\eta D_{1})^{2} + D_{1}(L\eta)^{2} + \frac{1}{2}D_{0}\right)\frac{4}{\sqrt{v_{0}}} + 2L\eta\ln v_{K}\bigg),
$
where $K$ denotes the total number of steps, and $v_{K}$ is accumulated sum of the squared gradient norm (see Algorithm~\ref{alg_god}).
\end{corollary}
The proof is presented in Appendix~\ref{append_section_tech}.
Consider the right-hand side of Cor.~\ref{corollary_convergence} as a function of $\eta$:
$
    h(\eta) := \frac{1}{\sqrt{v_{K}}}\left(\gamma_{1}\eta^{2} + \gamma_{2}\eta\ln{v_{K}} + \frac{\gamma_{3}}{\eta}\right),
$
where $\gamma_{1}, \gamma_{2},$ and $\gamma_{3}$ denote the corresponding problem-dependent values for simplification purpose.
We note that the accumulated gradient norm $v_{K}$ increases; therefore, $h(\eta)$ must decrease as $v_K$ increases to achieve tuning-free convergence. 
We discuss two possible cases of $\eta$ preserving the tuning-free convergence ability:
\begin{itemize}[noitemsep, leftmargin=1cm, topsep=0pt]
    \item Case 1: Suppose that $\eta$ is a constant value.
    \item Case 2: Suppose that $\eta$ is constant for a fixed $K$ but varies with $K$, and one possible solution is $\eta = (v_{K})^q$, where $0<q<\frac{1}{4}$ and $v_{K} > 1$, so $h\left(\eta\right) = \gamma_{1}(v_{K})^{2q-\frac{1}{2}} + \gamma_{2}(v_{K})^{q-\frac{1}{2}}\ln{v_{K}} + \gamma_{3}(v_{K})^{-q-\frac{1}{2}}$ is continually decreasing with the increasing of $v_{K}$.
\end{itemize}

Meanwhile, the following general update inequality follows directly from the descent lemma for smooth functions:
$
f(\rvx_{k+1}) \leq f(\rvx_{k}) - \eta_{k}\nabla f(\rvx_{k})\rvg_{k} + \eta^{2}_{k}L||\rvg_{k}||^{2}.
$
Since the right-hand side of the descent lemma forms a quadratic function w.r.t $\eta$,
the (worst-case) optimal progressive step-size $\eta^{\text{opt}} = \frac{1}{2}\eta^{\text{div}}$. Here, $\eta^{\text{div}}$ represents the step size that leads to divergence ($\forall \eta > \eta^{\text{div}}$ is a divergent step size).

%
%

Incorporating the concepts of preserving tuning-free convergence and achieving $0.5\times$ diverging step size under various settings,
we formally formulate the \textit{golden step size} of AdaGrad-Norm as
\begin{align}
\label{eq_defination_opt}
&\eta^{\text{gold}} = \frac{1}{2}\argmax_{\eta} \mathbb{E}_{x \in \mathbb{R}^{+}} h(x, \eta)\\
&\text{ s.t.} \lim_{x\rightarrow +\infty} \Big(h(x, \eta) = \frac{1}{\sqrt{x}}\big(\gamma_{1}\eta^{2} + \gamma_{2}\eta\ln{x} + \frac{\gamma_{3}}{\eta}\big)\Big) = 0\nonumber,
\end{align}
where $x:= v_{K}$ for simplification purposes. Here, the expectation over $x \in \mathbb{R}^{+}$ denotes various settings, the constraint $\lim_{x\rightarrow +\infty} h(x, \eta) = 0$ ensures preservation of tuning-free convergence, and $\argmax_{\eta} h(x, \eta)$ approximates the step size that causes the optimization to diverge, i.e., the potentially largest $h(x, \eta)$.
Please also refer to the discussion regarding incorporating the optimal progressive learning rate and the learning rate that converges with the training trajectory in Section~\ref{sec_conclusion}.

\subsection{Solution for The Golden Step Size}
\label{sub_sec_golden_step_size}
We now provide the analytical solution for~\eqref{eq_defination_opt}. 
First, we derive the domain of $\eta$ based on the constraint. Considering the constraint $\lim_{x\rightarrow +\infty} \frac{1}{\sqrt{x}}(\gamma_{1}\eta^{2} + \gamma_{2}\eta\ln{x} + \frac{\gamma_{3}}{\eta}) = 0$, to ensure $\eta$ satisfies the constraint, one straightforward approach is to consider $\lim_{x\rightarrow +\infty} \frac{\gamma_{1}\eta^{2}}{\sqrt{x}} = 0$, $\lim_{x\rightarrow +\infty} \frac{\gamma_{2}\eta \ln x}{\sqrt{x}} = 0$, and $\lim_{x\rightarrow +\infty} \frac{\gamma_{3}}{\eta \sqrt{x}} = 0$. This implies that the domain of $\eta$ is the intersection of subdomains where each sub-component $\{\frac{\gamma_{1}\eta^{2}}{\sqrt{x}}, \frac{\gamma_{2}\eta \ln x}{\sqrt{x}}, \frac{\gamma_{3}}{\eta \sqrt{x}}\}$ achieves 0 simultaneously. In other words, $\mathcal{O}(\eta) = \left(< \mathcal{O}(x^{\frac{1}{4}})\right) \cap \left(< \mathcal{O}(\frac{\sqrt{x}}{\ln x})\right) \cap \left(> \mathcal{O}(x^{-\frac{1}{2}})\right) = \left(< \mathcal{O}(x^{\frac{1}{4}})\right)  \cap \left(> \mathcal{O}(x^{-\frac{1}{2}})\right) = \mathcal{O}(x^t)$ with $t \in (-\frac{1}{2}, \frac{1}{4})$.
%
%
Therefore, we consider two cases as we discussed in Section~\ref{subsec_gold_ada}: when $\eta:= x^{t} \text{ where } t\in (-\frac{1}{2}, \frac{1}{4})$ (Case 2) and when $\eta$ is a constant value (Case 1), which covers the above domain.
We then compare the maximum expected value $\max\mathbb{E}[h(x, \eta)]$ across these two cases.

For $\eta:= x^{t}$ where $t\in (-\frac{1}{2}, \frac{1}{4})$ (Case 2), suppose that $x$ is bounded and uniformly distributed, i.e., $x \sim\mathcal{U}(C_{1}, C_{2})$, where $C_{2}\gg C_{1}>1$, we have
\begin{align*}
&\mathbb{E}_{x \sim \mathcal{U}(C_{1}, C_{2})}\frac{1}{\sqrt{x}}\left(\gamma_{1}\eta^{2} + \gamma_{2}\eta\ln{x} + \frac{\gamma_{3}}{\eta}\right) \\
&= \frac{1}{C_{2} - C_{1}}\int_{C_{1}}^{C_{2}}\frac{1}{\sqrt{x}}(\gamma_{1}x^{2t} + \gamma_{2}x^{t}\ln{x} + \gamma_{3}x^{-t}) dx\\
&\nonumber = \frac{1}{C_{2} - C_{1}}\bigg(\frac{\gamma_{1}}{\frac{1}{2} + 2t}x^{\frac{1}{2} + 2t}  + \frac{\gamma_{2}}{\frac{1}{2} + t}x^{\frac{1}{2} + t}\ln x \nonumber\\
&\qquad - \frac{\gamma_{2}}{(\frac{1}{2} + t)^2}x^{\frac{1}{2} + t} + \frac{\gamma_{3}}{\frac{1}{2} - t}x^{\frac{1}{2} - t}\bigg)\bigg|_{C_{1}}^{C_{2}}\nonumber\\
& \approx \mathcal{O}\left(\frac{C_{2}^{\frac{1}{2} + 2t} - C_{1}^{\frac{1}{2} + 2t}}{C_{2} - C_{1}}\right) \approx \mathcal{O}(C_{2}^{-\frac{1}{2} + 2t}). 
\end{align*}

Since $t\in (-\frac{1}{2}, \frac{1}{4})$, it is straightforward to observe that $\eta = \lim_{t\rightarrow \frac{1}{4}^{-}}x^t$ attains highest expectation value with $\mathbb{E}_{x \sim (C_{1}, C_{2})}[h(x, \eta)] \approx \mathcal{O}(C_{2}^{-\frac{1}{2} + 2\frac{1}{4}}) = \mathcal{O}(1)$.

Given $\eta$ is constant value (Case 1), and supposing $x$ is bounded and uniformly distributed, i.e., $x \sim\mathcal{U}(C_{1}, C_{2})$, where $C_{2}\gg C_{1}>1$, we have
\begin{align*}
&\mathbb{E}_{x \sim \mathcal{U}(C_{1}, C_{2})}\frac{1}{\sqrt{x}}\left(\gamma_{1}\eta^{2} + \gamma_{2}\eta\ln{x} + \frac{\gamma_{3}}{\eta}\right) \nonumber\\
&= \frac{1}{C_{2} - C_{1}}\int_{C_{1}}^{C_{2}}\frac{1}{\sqrt{x}}\left(\gamma_{1}\eta^{2} + \gamma_{2}\eta\ln{x} + \frac{\gamma_{3}}{\eta}\right) dx\nonumber\\
& = \frac{1}{C_{2} - C_{1}}\Big(2(\gamma_{1}\eta^{2} + \frac{\gamma_{3}}{\eta})\sqrt{x}\big|_{C_{1}}^{C_{2}} \\
&\qquad + \gamma_{2}\eta(2\sqrt{x}\ln x - 4\sqrt{x})\big|_{C_{1}}^{C_{2}}\Big) \nonumber\\
& \approx \mathcal{O}\left(\frac{\ln(C_{2}) - \ln(C_{1})}{\sqrt{C_{2} - C_{1}}}\right) \approx \mathcal{O}\left(\frac{\ln(C_{2})}{\sqrt{C_{2}}}\right).
\end{align*}

With $C_2 \gg 1$, it follows that $\eta$ is a constant value attaining highest expectation value with $\mathbb{E}_{x \sim \mathcal{U}(C_{1}, C_{2})}[h(x, \eta)] \approx \mathcal{O}(\frac{\ln(C_{2})}{\sqrt{C_{2}}})$.

Since $\eta^{\text{gold}}$ desires the maximum expectation and $\mathcal{O}_{\text{Case 1}}(\frac{\ln(C_{2})}{\sqrt{C_{2}}}) \ll \mathcal{O}_{\text{Case 2}}(1)$, we conclude that $\eta^{\text{gold}} = \frac{1}{2}\lim_{t\rightarrow \frac{1}{4}^{-}}x^{t}$, where $\rightarrow1/4^{-}$ means that $t$ approaches $1/4$ from below, is the desired golden step size that achieves the maximum expectation under the defined constraint. 


\subsection{Scale-Free Property of Golden Step Size}
\label{subsec_background}
We adopted the definition of the scale-free property of an optimization method from~\cite{khaled2023dowg}, where it is defined as \textit{multiplying $f$ by a constant factor $\alpha > 0$ and minimizing $\alpha f$ does not change the method’s trajectory at all}.
Here, following prior work, we use parameter-free to mean that no task-specific learning-rate tuning is required; such methods may still use fixed global constants. Scale-freeness is a desirable property of these methods~\citep{khaled2023dowg,defazio2023learning,mishchenko2023prodigy}

Taking Theorem~\ref{theorem_ngd} as an example, also appearing in~\cite{khaled2023dowg, yang2024two}, to illustrate the concept, Normalized Gradient Descent (NGD) is inherently scale-free, as rescaling $f$ to $\alpha f$ does not alter the step size trajectory, i.e., $\eta^{\text{opt}} = D_{0}/\sqrt{K}$ remains unchanged before and after rescaling. Meanwhile, if we can approximate $D_{0}$ dynamically, then NGD qualifies as parameter-free.
We summarize the following key takeaways:
(1). An immediate observation regarding scale-free methods is that the derived step size is not correlated with the scale $\alpha$. In terms of the parameter-free methods, the corresponding step size does not depend on terms such as the scale $\alpha$ or problem properties that are unknown or cannot be approximated;
(2). It is important to note that parameter-free does not imply the ability to arbitrarily scale the derived step size.

Particularly, there is an immediate observation that $\eta^{\text{gold}}$ is independent of problem-dependent values $\gamma_{1}, \gamma_{2}$, and $\gamma_{3}$, further reinforcing the notion that rescaling function will not alter its trajectory.

\begin{theorem}[Example adopted from \cite{levy2017online, grimmer2019convergence}]
\label{theorem_ngd}
Suppose that $f$ is convex (bounded below by $f^{\star}:=f(\rvx^{\star})$) and satisfies Assumption~\ref{ass_lsmooth}. If we run NGD $\rvx_{k+1} = \rvx_{k} - \eta\frac{\nabla f(\rvx_k)}{||\nabla f(\rvx_k)||}$, we have
$
\min_{k=0,\cdots,K-1}(f(\rvx_{k}) - f^{\star})\leq \frac{L}{2}(\frac{D_{0}^{2}}{2\eta K} + \frac{\eta}{2})^{2},
$
where $D_{0}:=||\rvx_{0} - \rvx^{\star}||$, and $\eta^{\text{opt}} = \frac{D_{0}}{\sqrt{K}}$.
\end{theorem}

\subsection{Algorithm: AdaGrad-Norm-Based Parameter-Free Optimizer}
\label{subsec_alg1}

Our analysis shows that when the \textit{golden step size} is used, the updated AdaGrad-Norm algorithm is expected to preserve tuning-free convergence and approximate the optimal step size across various settings.
We hereby propose a novel parameter-free optimization algorithm that integrates the golden step size into AdaGrad-Norm.

Since $\eta^{\text{gold}} = \frac{1}{2}\lim_{t\rightarrow \frac{1}{4}^{-}}x^{t}$, we define a numerator function $s(x): =px^{q}$, where $p \rightarrow 1/2$ and $q\rightarrow 1/4^{-}$, to represent the embedding of the golden step size.
The proposed parameter-free training method, named \textsc{GOG} (Golden step size over Gradients), is summarized in Algorithm~\ref{alg_god}. 
Note that Algorithm~\ref{alg_god} approximates $v_{K}$ by $v_{k}$, and a similar idea is the optimality gap approximation exploited in the baseline method DoG and its variants~\citep{ivgi2023dog}. Because of space limitations, please refer to Appendix~\ref{append_subsec_dynamic_eta} for a discussions of the approximation.

\begin{algorithm}[!t]
   \caption{\textsc{GOG} based on AdaGrad-Norm}
   \label{alg_god}
\begin{algorithmic}
   \STATE {\bfseries Input:} parameter $\rvx_{1}$, $\eta_{k}$ (default 1), objective function $f(\rvx)$, $p, q$
   \STATE Initialize $v_{1} = 0$
   \STATE $s(x) = px^{q}$
   \FOR{$k=1$ {\bfseries to} $K$}
   \STATE $\rvg_{k} \in \partial f(\rvx_{k}, \xi_{k})$
   \STATE $v_{k+1} = v_{k} + ||\rvg_{k}||^{2}$
   \STATE $r_{k+1} = s(v_{k+1})$
   \STATE $\rvx_{k+1} = \rvx_{k} - \eta_{k}\frac{\text{$r_{k+1}$}}{\sqrt{v_{k+1}}}\rvg_{k}$
   \ENDFOR
\end{algorithmic}
\end{algorithm}

\begin{algorithm}[!ht]
   \caption{\textsc{AdamG} based on Adam}
   \label{alg_adamg}
\begin{algorithmic}
   \STATE {\bfseries Input:} parameter $x_{1}$, $\eta_{k}$ (default 1), $p, q$, $\beta_{1}$, $\beta_{2}$, $\beta_{3}$, $\epsilon$. 
   \STATE Initialize $m_{1} = 0, v_{1} = 0$, $r_{1} = 0$
   \STATE  $s(x) = px^{q}$
   \FOR{$k=1$ {\bfseries to} $K$}
   \STATE $g_{k} \in \partial f(x_{k}, \xi_{k})$
   \STATE $v_{k+1} = \beta_{2} v_{k} + (1-\beta_{2}) g_{k}^{2}$
   \STATE $\hat{v}_{k+1} = v_{k+1}/(1 - \beta_{2}^{k})$
   \STATE $r_{k+1} = \beta_{3} r_{k} + (1-\beta_{3})s(v_{k+1})$
   \STATE $m_{k+1} = \beta_{1} m_{k} + (1-\beta_{1}) r_{k+1} g_{k}$
   \STATE $\hat{m}_{k+1} = m_{k+1}/(1 - \beta_{1}^{k})$
   
   \STATE $x_{k+1} = x_{k} - \min(\eta_{k}, 1/\sqrt{k})\frac{\hat{m}_{k+1}}{(\sqrt{\hat{v}_{k+1}} + \epsilon)}$
   \ENDFOR
\end{algorithmic}
\end{algorithm}


\subsection{Algorithm: Adam-Based Parameter-Free Optimizer}
\label{subsec_alg2}

In addition to \textsc{GOG}, we further develop an Adam-like method incorporating the golden step size, leading to a practical parameter-free optimizer with momentum acceleration. 
Similarly, we use a numerator function $s(\cdot)$ to encode the golden step size. 
We then compute an exponential moving average (EMA) of the golden step size as follows:
$
    r_{k+1} = \beta_{3} r_{k} + (1-\beta_{3})s(v_{k+1}),
$
where $\beta_{3} \in [0,1)$ is the exponential decay rate.

In addition, inspired by D-Adapt Adam~\citep{defazio2023learning}, we use EMA golden step size, $r_{k+1}$, in the first-moment estimate instead of the raw coefficient in the parameter update.
Additionally, the term $1/\sqrt{k}$ in Algorithm~\ref{alg_adamg} is a commonly adopted strategy appearing in optimizing stochastic problems against error caused by randomness in gradient estimation~\citep{nesterov2018lectures, ge2015escaping}. 
We treat it as a practical heuristic.
%
The parameter-free optimizer \textsc{AdamG}, which incorporates the golden step size into Adam, is summarized in Algorithm~\ref{alg_adamg}.
Note that $\beta_{1}$, $\beta_{2}$, $\beta_{3}$, and $\epsilon$ in Algorithm~\ref{alg_adamg} have default values of 0.95, 0.999, 0.95, and $10^{-8}$, respectively. 
%


\section{Experiments}
\label{sec_exp}

\subsection{Reliability as an Evaluation Criterion}
\label{subsec_exp_criteria}
Traditional evaluation criteria for parameter-free optimizers primarily focus on convergence speed (e.g., loss curves) and solution quality (e.g., test accuracy comparison) for each task from a set of optimization tasks. These metrics are standard in both classical optimizer studies~\citep{kingma2014adam, liu2023sophia} and recent parameter-free approaches~\citep{ivgi2023dog, defazio2023learning, khaled2023dowg, mishchenko2023prodigy}.
However, these criteria overlook a key aspect of parameter-free optimizers—their ability to consistently perform well across a wide range of tasks. To address this gap, we introduce a novel metric called reliability, designed to more comprehensively assess the effectiveness of parameter-free optimizers.

We begin by categorizing all conducted optimization tasks based on the baseline Adam optimizer with different learning rates: Adam(1e-2), Adam(1e-3), Adam(1e-4), and Adam(1e-5). This classification is supported by both practical and theoretical considerations:
\begin{itemize}
\item \textit{Practical intuition:} Adam with varying learning rates is the de facto baseline for most optimization tasks.
\item \textit{Theoretical intuition:} Many theoretical frameworks demonstrate that a fixed optimization method with a tunable learning rate can adapt to a wide variety of functions and gradient oracle settings.
\end{itemize}
Based on this setup, we formally define the reliability metric as follows:
\begin{definition}[Reliability]
Let $\mathcal{G}=\{10^{-2},10^{-3},10^{-4},10^{-5}\}$ be the
evaluated Adam learning-rate grid. For each task $i$, let $P_i(M)$
denote the mean performance of method $M$ over three seeds, and let
$P_i^\star=\max_{\eta\in\mathcal{G}}P_i(\operatorname{Adam}(\eta))$.
We assign task $i$ to the group $\mathcal{T}_{\eta}$ corresponding to
its best-performing Adam learning rate. For a tolerance $\delta$
measured in absolute performance points, reliability is
\[
R(M;\delta)=\frac{1}{4}\sum_{\eta\in\mathcal{G}}
\frac{1}{|\mathcal{T}_{\eta}|}
\sum_{i\in\mathcal{T}_{\eta}}
\mathbb{I}\!\left[P_i(M)\geq P_i^\star-\delta\right].
\]
Thus, reliability is the macro-average of the four group-wise success
rates rather than the fraction of all individual tasks. We use
$\delta=5$ points as the main setting and also report results for
$\delta\in\{1,10\}$.
\end{definition}

\begin{table*}[!ht]
    \centering
    \small
    \begin{threeparttable}
    \setlength{\tabcolsep}{3pt}
    \begin{tabular*}{\textwidth}{@{\extracolsep{\fill}} l *{10}{c} }
    \toprule
    \multirow{2}{*}{Dataset} & \multirow{2}{*}{Algorithm$^{[a]}$ } & \multicolumn{4}{c}{Epoch 20, pretrained network$^{[b]}$} & \multicolumn{4}{c}{Epoch 100, randomly init. network} \\
    \cmidrule(l{4pt}r{4pt}){3-6}
    \cmidrule(l{4pt}r{12pt}){7-10}
     & & DenseNet & ResNet & ViT-B &  VGG & DenseNet & ResNet & ViT-B &  VGG \\
    \cmidrule(l{0pt}r{12pt}){1-10}
    \multirow{7}{*}{CIFAR-10} & Adam$^{\star}$ &   88.1±0.2 &  92.4±0.4 & 77.3±0.4 & 84.6±0.2 &  79.6±1.1 & 85.3±0.7 &  56.3±0.7 & 77.3±0.5\\
    \cmidrule(l{4pt}r{10pt}){2-10}
    & DoG &  78.4±0.8$^{\times}$ &  88.3±0.7&  63.7±0.7$^{\times}$ & 80.5±1.8 &  62.2±0.2$^{\times}$ &  71.1±0.7$^{\times}$ & 54.8±0.4 & 72.9±0.2\\
    & DoWG &  80.4±0.4$^{\times}$ &  86.4±0.7$^{\times}$ & 67.1±1.2$^{\times}$ & 80.7±1.2 & 53.9±0.5$^{\times}$ &  65.5±0.4$^{\times}$ & 50.8±0.9$^{\times}$ & 52.7±3.1$^{\times}$\\
    & D-Adapt&  88.2±0.1 &  91.6±0.4 & 77.3±1.1& 71.2±10.2$^{\times}$ & 72.3±0.3$^{\times}$ & 83.3±0.3 & 11.3±1.2$^{\times}$ & 49.1±27.7$^{\times}$\\
    & Prodigy &  87.4±0.1 &  90.9±0.5 & 79.5±0.2 & 86.1±0.2 & 64.0±0.6$^{\times}$ & 73.7±0.1$^{\times}$ & 21.1±8.2$^{\times}$ & 75.5±0.6\\
    & \colorbox{gray!30}{\textsc{AdamG}} &  86.1±0.3 & 91.1±0.4 & 78.6±0.4 & 87.3±0.0 & 68.1±0.6$^{\times}$ & 75.9±0.6$^{\times}$ & 58.1±0.3 & 77.4±0.4\\
    \cmidrule(l{0pt}r{12pt}){1-10}
    \multirow{7}{*}{CIFAR-100} & Adam$^{\star}$ &  65.2±0.2 & 72.8±0.8 & 51.1±0.4 & 60.1±0.3 & 47.2±1.5 & 57.5±0.4 & 27.8±0.3 & 33.5±0.5\\
    \cmidrule(l{4pt}r{10pt}){2-10}
    & DoG &   50.6±2.5$^{\times}$ & 69.0±2.7 & 30.7±2.7$^{\times}$ & 56.4±0.2 & 33.6±0.3$^{\times}$ & 47.4±0.7$^{\times}$ & 29.2±0.3 & 31.8±1.2 \\
    & DoWG &  55.7±0.1$^{\times}$ & 65.3±3.1$^{\times}$ & 40.4±1.3$^{\times}$ & 56.2±0.4 & 26.7±0.4$^{\times}$ & 38.1±0.5$^{\times}$ & 24.9±0.3 & 1.0±0.0$^{\times}$ \\
    & D-Adapt & 65.4±0.0 & 71.8±1.0 & 53.6±1.0 & 43.0±5.3$^{\times}$ & 43.7±0.6 & 55.7±0.8 & 1.0±0.1$^{\times}$ & 29.2±0.3 \\
    & Prodigy & 64.4±0.1 & 72.1±0.8 & 55.9±0.3 & 62.4±0.5 & 42.0±0.2$^{\times}$ & 53.7±0.7 & 5.7±1.6$^{\times}$ & 41.2±0.6 \\
    & \colorbox{gray!30}{\textsc{AdamG}} & 62.6±0.2 & 70.4±1.3 & 54.5±0.1 & 63.1±0.1 & 35.4±0.0$^{\times}$ & 44.9±0.4$^{\times}$ & 31.5±0.2 & 42.1±0.4 \\
    \cmidrule(l{0pt}r{12pt}){1-10}
    \multirow{7}{*}{\shortstack{Tiny\\ImageNet}} & Adam$^{\star}$ &  62.9±0.2 & 63.0±0.2 & 57.3±0.4 & 59.6±0.2 & 39.2±0.1 & 50.9±0.5 & 16.4±0.2 & 35.2±1.0\\
    \cmidrule(l{4pt}r{10pt}){2-10}
    & DoG &  61.4±0.3 & 69.1±2.0 & 49.5±0.9$^{\times}$ & 57.4±1.7 & 34.5±0.3$^{\times}$ & 45.5±0.5$^{\times}$ & 14.2±0.3 & 24.9±1.0$^{\times}$\\
    & DoWG &  60.7±0.3 & 61.1±3.9 & 45.3±2.1$^{\times}$ & 57.2±0.1 & 24.4±0.1$^{\times}$ & 28.5±0.4$^{\times}$ & 15.5±0.1 & 7.8±4.5$^{\times}$ \\
    & D-Adapt &  60.2±0.2 & 60.3±0.8 & 64.5±0.3 & 23.1±7.4$^{\times}$ & 36.0±0.1 & 47.3±0.4 & 1.2±0.9$^{\times}$ & 27.4±0.5$^{\times}$ \\
    & Prodigy &  62.0±0.2 & 63.6±1.0 & 63.2±0.3 & 58.8±0.0 & 40.5±0.4 & 53.6±0.4 & 8.1±1.1$^{\times}$ & 33.8±0.1 \\
    & \colorbox{gray!30}{\textsc{AdamG}} &  62.7±0.1 & 64.2±1.4 & 60.2±0.3 & 59.7±0.3 & 26.3±0.4$^{\times}$ & 39.0±0.2$^{\times}$ & 16.9±0.1 & 35.9±0.4\\
    \bottomrule
    \end{tabular*}  
    \begin{tablenotes}
        \item $^{[a]}$ Adam$^{\star}$ denotes the best setting selected from: Adam(1e-2), Adam(1e-3), Adam(1e-4), and Adam(1e-5).
        \item $^{[b]}$ $^{\times}$ denotes that the performance measure of the specific parameter-free optimizer is at least 5\% lower than that of the best Adam variant.
    \end{tablenotes}
    \caption{Test accuracy (\%) on CIFAR-10, CIFAR-100, and Tiny-ImageNet over three random seeds. }
    \label{tab_cifar10}
\end{threeparttable}
\end{table*}

\begin{table*}[!ht]
    \centering
    \small
    \begin{threeparttable}
    \begin{tabular*}{\textwidth}{@{\extracolsep{\fill}} l *{9}{c} }
    \toprule
     Model& Algorithm & SST-2 & MRPC & QQP & MNLI & QNLI &  RTE & WNLI\\
     \cmidrule(l{4pt}r{4pt}){3-9}
    \cmidrule{1-10}
    \multirow{7}{*}{BERT} & Adam$^{\star}$ & 92.5±0.3 & 83.2±0.9 & 90.7±0.1 & 84.1±0.1 & 91.3±0.3 & 65.8±1.2 & 52.1±6.0 \\
    \cmidrule(l{4pt}r{4pt}){2-9}
    & DoG &  91.4±0.3 & 74.3±4.2$^{\times}$ & 89.1±0.0 & 83.1±0.2 & 90.6±0.1 & 51.9±3.3$^{\times}$ & 57.3±1.3 \\
    & DoWG &  74.8±17.3$^{\times}$ & 72.3±2.6$^{\times}$ & 79.5±11.5$^{\times}$ & 59.5±20.6$^{\times}$ & 74.6±17.8$^{\times}$ & 51.1±2.7$^{\times}$ & 52.1±6.0\\
    & D-Adapt & 76.6±18.1$^{\times}$ & 68.4±0.0$^{\times}$ & 63.2±0.0$^{\times}$ & 66.1±24.2$^{\times}$ & 73.9±17.3$^{\times}$ & 61.3±9.9 & 52.1±6.0 \\
    & Prodigy & 91.5±1.3 & 73.5±7.5$^{\times}$ & 90.4±0.2 & 83.1±0.5 & 90.8±0.1 & 65.8±3.5 & 46.5±13.9$^{\times}$ \\
    & \colorbox{gray!30}{\textsc{AdamG}} & 90.9±0.4 & 81.5±3.3 & 90.4±0.0 & 83.9±0.4 & 89.8±0.3 & 65.2±3.5 & 52.1±6.0 \\
    \cmidrule{1-9}
    \multirow{7}{*}{\shortstack{GPT-2\\with \\LoRA}} & Adam$^{\star}$ & 90.8±0.2  & 76.1±0.9 & 86.1±0.0 & 78.8±0.1 & 84.9±0.4  & 61.6±2.2  & 48.4±5.7 \\
    \cmidrule(l{4pt}r{4pt}){2-9}
    & DoG & 64.2±9.3$^{\times}$ & 67.9±0.2$^{\times}$ & 67.8±2.0$^{\times}$ & 43.8±0.0$^{\times}$ & 51.7±2.0$^{\times}$ & 50.1±1.3$^{\times}$ & 48.8±5.4  \\
    & DoWG & 90.4±0.9 & 69.8±0.6$^{\times}$ & 81.5±0.3$^{\times}$ & 72.8±0.4$^{\times}$ & 81.9±0.9$^{\times}$ & 50.8±3.9$^{\times}$ & 46.9±6.7$^{\times}$ \\
    & D-Adapt & 55.2±4.4$^{\times}$ & 58.5±14.0$^{\times}$ & 65.0±1.3$^{\times}$ & 32.8±0.1$^{\times}$ & 50.2±0.5$^{\times}$ & 50.8±1.8$^{\times}$ & 50.7±5.3 \\
    & Prodigy & 85.7±1.9$^{\times}$ & 68.6±0.3$^{\times}$ & 64.6±1.0$^{\times}$ & 33.1±0.5$^{\times}$ & 52.2±1.4$^{\times}$ & 50.9±2.8$^{\times}$ & 51.6±5.7 \\
    & \colorbox{gray!30}{\textsc{AdamG}} & 90.9±0.6 & 72.5±1.4& 85.6±0.1 & 78.8±0.1 & 86.0±0.5 & 58.0±4.9 & 49.8±5.8 \\
    \bottomrule
    \end{tabular*}
    \caption{Performance (\%) of fine-tuning pretrained BERT and GPT-2 with
    LoRA on the GLUE benchmark at epoch 3 over three random seeds.}
    \label{tab_bert}
\end{threeparttable}
\end{table*}

\subsection{Setup}
We compare \textsc{AdamG} to DoG~\citep{ivgi2023dog}, DoWG~\citep{khaled2023dowg}, D-Adapt (Adam)~\citep{defazio2023learning}, and Prodigy (Adam)~\citep{mishchenko2023prodigy} with evaluation criteria reliability, solution quality, and convergence speed. 
Unless otherwise specified, all Adam and Adam-type parameter-free optimizers are paired with a cosine learning rate scheduler.
That is, the default value of $\eta_{k}$ in \textsc{AdamG}, D-Adapt Adam, and Prodigy Adam is set to 1 with an additional cosine-annealing schedule, following the default choice of previous work~\citep{defazio2023learning, mishchenko2023prodigy}. 

The optimization tasks cover two main categories:
For image tasks, we fully fine-tune pretrained DenseNet121~\citep{huang2017densely}, ResNet18~\citep{he2016deep}, ViT-B/16~\citep{dosovitskiy2020vit}, and VGG11~\citep{simonyan2014very}, and train the same architectures from random initialization on CIFAR-10, CIFAR-100, and Tiny-ImageNet~\citep{krizhevsky2009learning, russakovsky2015imagenet}. For language tasks, we fully fine-tune pretrained BERT~\citep{devlin2018bert} and perform LoRA fine-tuning of GPT-2~\citep{radford2019language, hu2021lora} on the GLUE benchmark~\citep{wang2018glue}.
Note that we use the numerator function $s(x) = 0.2x^{0.24}$ for all optimization tasks.

Other setup details of the 38 tasks are summarized in Appendix~\ref{append_section_exp_setup}.

\subsection{Performance Comparison}
\label{subsec_comp}
We summarize the average performance across 38 optimization tasks in Table~\ref{tab_cifar10} (vision tasks) and Table~\ref{tab_bert} (language tasks). Due to space constraints, the complete versions of these tables are provided in the Appendix.
A cursory examination of the results suggests that \textsc{AdamG} performs consistently well across most tasks, while methods such as DoWG frequently underperform. However, direct comparisons based on raw table entries can be ambiguous and difficult to quantify. To facilitate a more systematic evaluation, we introduce two performance metrics, reliability and solution quality, derived from Table~\ref{tab_cifar10} and Table~\ref{tab_bert}, and summarized in Table~\ref{tab_stability}.

\begin{table*}[!ht]
    \centering
    \small
    \begin{threeparttable}
    \begin{tabular*}{\textwidth}{@{\extracolsep{\fill}} l *{7}{c} }
    \toprule
    \multirow{2}{*}{Metrics}  & \multirow{2}{*}{Algorithm} & \multicolumn{4}{c}{Task category} & \multirow{2}{*}{Averaged}\\
    \cmidrule(l{4pt}r{4pt}){3-6}
    &  & Adam(1e-2) & Adam(1e-3) & Adam(1e-4) &  Adam(1e-5) & result\\
    \cmidrule{1-7}
    \multirow{5}{*}{Reliability ratio} & DoG  &  2/5  &  6/12  &  7/15  &  4/6    & 0.50\\
    & DoWG  &  2/5  &  2/12  &  8/15  &  0/6    & 0.27\\
      & D-Adapt Adam     &  4/5  &  10/12  &  3/15  &  1/6    & 0.50\\
      & Prodigy Adam &  1/5  &  11/12  &  7/15  &  5/6   & 0.60\\
      & \colorbox{gray!30}{\textsc{AdamG}} &  2/5  &  9/12  &  15/15  &  6/6   & \textbf{0.78}\\
      \cmidrule{1-7}
       \multirow{5}{*}{Solution quality (\%)} & DoG     & 9.0 & 6.2 & 12.4 & 4.5  & 8.0\\
      &  DoWG    &   13.5 & 11.2 & 8.3 & 15.9 & 12.2 \\
      &  D-Adapt Adam     &2.6 & 5.1 & 20.9 & 16.4& 11.2  \\
       & Prodigy Adam &7.6 & 1.5 & 12.2 & 2.1&  5.8 \\
       & \colorbox{gray!30}{\textsc{AdamG}} & 6.5 & 4.1 & 0.3 & 1.0 &  \textbf{3.0} \\
    \bottomrule
    \end{tabular*}    
    \caption{Performance comparison: \textit{Reliability} and \textit{Solution quality}.}
    \label{tab_stability}
\end{threeparttable}
\end{table*}

\begin{table*}[!ht]
    \centering
    \small
    \begin{threeparttable}
    \begin{tabular*}{\textwidth}{@{\extracolsep{\fill}} l *{7}{c} }
    \toprule
     Metrics  & Algorithm & Adam(1e-2) & Adam(1e-3) & Adam(1e-4) &  Adam(1e-5) & Avg.\\
    \cmidrule{1-7}
    \multirow{7}{*}{1\% Reliability ratio} & DoG  &  2/5  &  1/12  &  1/15  &  4/6  & 0.30 & \\
    & DoWG  & 2/5  &  0/12  &  4/15  &  0/6 &   0.17& \\
    & \textsc{GOG}  &  2/5  &  0/12  &  3/15  &  4/6  &  0.31& \\
      & D-Adapt Adam  &  2/5  &  4/12  &  2/15  &  1/6 &  0.25 & \\
      & Prodigy Adam &  1/5  &  8/12  &  6/15  &  5/6  &  0.52& \\
      & \textit{Adam(1e-3)} &  2/5  &  12/12  &  5/15  &  0/6   & 0.43& \\
      & \colorbox{gray!30}{\textsc{AdamG}} &  2/5  &  5/12  &  15/15  &  6/6   & 0.70& \\
      \cmidrule{1-7}
      \multirow{7}{*}{10\% Reliability ratio} & DoG  &  2/5  &  9/12  &  7/15  &  4/6   & 0.57& \\
    & DoWG  & 2/5  &  6/12  &  9/15  &  0/6   & 0.37& \\
    & \textsc{GOG}  &  2/5  &  4/12  &  13/15  &  6/6  &  0.65& \\
      & D-Adapt Adam  &  5/5  &  10/12  &  4/15  &  1/6 &  0.56 & \\
      & Prodigy Adam &  2/5  &  11/12  &  8/15  &  5/6  & 0.67 & \\
      & \textit{Adam(1e-3)} &  4/5  &  12/12  &  9/15  &  0/6    &0.60& \\
      & \colorbox{gray!30}{\textsc{AdamG}} &  3/5  &  9/12  &  15/15  &  6/6  &  0.83& \\
    \bottomrule
    \end{tabular*}    
    \caption{\textit{Reliability} at 1\% and 10\% performance-gap thresholds.}
    \label{tab_stability_virous_gap}
\end{threeparttable}
\end{table*}

\begin{table*}[!ht]
    \centering
    \small
    \begin{threeparttable}
    \begin{tabular*}{\textwidth}{@{\extracolsep{\fill}} l *{9}{c} }
    \toprule
    \multirow{3}{*}{Algorithm}  & \multicolumn{8}{c}{Test accuracy (\%) under CIFAR-10} \\
    & \multicolumn{4}{c}{Epoch 20, pretrained network} & \multicolumn{4}{c}{Epoch 100, randomly initialized network} \\
    \cmidrule(l{4pt}r{4pt}){2-5}
    \cmidrule(l{4pt}r{4pt}){6-9}
     & DenseNet & ResNet & ViT-B &  VGG & DenseNet & ResNet & ViT-B &  VGG \\
    \cmidrule{1-9}
    \textsc{AdamG} &  86.1 & 91.1 & 78.6 & 87.3 & 68.1 & 75.9 & 58.1 & 77.4\\
    \textsc{AdamG} ($\eta_{k} = 1$) & 85.6  & 91.1 & 77.6 & 86.7 & 68.3 & 75.6 & 57.1 & 76.1\\
    \bottomrule
    \end{tabular*}  
    \caption{Test accuracy on CIFAR-10 over three random seeds. }
    \label{tab_cifar10_robust_decay}
\end{threeparttable}
\end{table*}

\begin{table*}[!ht]
    \centering
    \small
    \begin{threeparttable}
    \begin{tabular*}{\textwidth}{@{\extracolsep{\fill}} l *{9}{c} }
    \toprule
    \multirow{4}{*}{Algorithm} & \multicolumn{9}{c}{Fine-tuning pretrained BERT under GLUE benchmark \& Epoch 3} \\
    \cmidrule(l{4pt}r{4pt}){2-10}
     & CoLA & SST-2 & MRPC & STS-B & QQP & MNLI & QNLI &  RTE & WNLI\\
     \cmidrule(l{4pt}r{4pt}){2-10}
     & Matthews corr.  & Acc. & F1 & Pearson corr.  & F1 & Matched acc.  & Acc. &  Acc. & Acc.\\
    \cmidrule{1-10}
    \textsc{AdamG} & 50.6 & 90.9 & 87.0& 88.7 & 87.1 & 83.9 & 89.8 & 65.2 & 52.1 \\
    \textsc{AdamG} ($\eta_{k} = 1$) & 49.9 & 90.3 & 87.7 & 88.5 & 87.1 & 83.5 & 90.0 & 67.3 & 52.1\\
    \bottomrule
    \end{tabular*}
    \caption{Performance of fine-tuning pretrained BERT on GLUE benchmark over three random seeds.}
    \label{tab_bert_robust_decay}
\end{threeparttable}
\end{table*}

\textbf{Reliability.}
Specifically, Table~\ref{tab_stability} summarizes the reliability ratios. 
As an example, consider the prior state-of-the-art, Prodigy:

(a). The entry $1/5$ under the “Adam(1e-2)” category indicates that out of five tasks where Adam(1e-2) achieves the best performance, Prodigy delivers comparable results (within 5\% of Adam) on only one. This suggests limited adaptability to tasks requiring large learning rates.

(b). The overall reliability score of 0.60 reflects that Prodigy achieves near-optimal performance in approximately 60\% of all tasks, on macro-average.

As shown in Table~\ref{tab_stability}, the proposed \textsc{AdamG} attains the highest average reliability ratio of 0.78, substantially outperforming existing baselines. This marks an 18-percentage-point improvement over the second-best method (Prodigy). Sensitivity analyses under alternative performance thresholds (1\% and 10\%) are reported in Table~\ref{tab_stability_virous_gap}, where \textsc{AdamG} continues to outperform others by significant margins: improving on the best baseline reliability from 0.52 to 0.70 under the 1\% gap and from 0.67 to 0.83 under the 10\% gap.

\textbf{Solution Quality.}
In addition to reliability, we measure a generalized version of solution quality to capture the average shortfall in performance of a parameter-free method relative to the best Adam variant. Specifically, we define $\text{solution quality} = \frac{1}{n}\sum_{i = 1}^{n}\max(\text{Perf}_{i}^{\text{best Adam}} - \text{Perf}_{i}^{\text{parameter-free}}, 0)$,  where $\text{Perf}_{i}$ denotes the performance metric (e.g., test accuracy) for task $i$, and $n$ is the total number of tasks.
This metric penalizes performance gaps relative to the optimal baseline and reflects an optimizer’s overall robustness. As shown in Table~\ref{tab_stability}, the proposed \textsc{AdamG} again achieves the best solution quality (3.0), meaning it consistently trails the best Adam variant by a smaller margin than all other competitors.

\textbf{Convergence speed.}
Figures~\ref{app_fig_resnet10}, \ref{app_fig_resnet100}, and \ref{app_fig_TINYIMAGENET} in the Appendix illustrate training loss curves for models—DenseNet121, ResNet18, ViT-B/16, and VGG11—across CIFAR-10, CIFAR-100, and Tiny-ImageNet datasets. Similarly, Figure~\ref{fig_lan_exp_comb} shows convergence behavior for full fine-tuning of BERT and LoRA-tuned GPT-2 on selected GLUE benchmark tasks.
Across all vision and language benchmarks, \textsc{AdamG} demonstrates competitive or superior convergence speed compared to all baseline methods.

\subsection{Ablation Study}

\textbf{Robustness to LR decay strategy}. Recall that $\eta_{k}$ in \textsc{AdamG} has a default value of 1 with an additional cosine annealing decay strategy, following the default choice of previous work~\cite{defazio2023learning, mishchenko2023prodigy}. 
Our empirical results show that \textsc{AdamG} is robust to the decay strategy.
Tables~\ref{tab_cifar10_robust_decay} and ~\ref{tab_bert_robust_decay}  illustrate the performance difference between default \textsc{AdamG} and modified \textsc{AdamG} (with a constant $\eta_{k}$). 
The two methods show no noticeable performance gap.

\textbf{Scaling of the golden step size}
Our analysis in Section~\ref{sub_sec_golden_step_size} derives  $\eta^{\text{gold}} = \frac{1}{2}\lim_{t\rightarrow \frac{1}{4}^{-}}x^{t}$ and further suggests a numerator function $s(x) =px^{q}$, where $p \rightarrow 1/2, q\rightarrow 1/4^{-}$. 
Specifically, we employ $s(x) = 0.2x^{0.24}$ for all optimization tasks in the experiments. In the following, we verify the effectiveness of scaling $p$.

Table~\ref{append_tab_cifar10} and Table~\ref{append_tab_bert} (Appendix) demonstrate the performance comparisons between the default \textsc{AdamG} and \textsc{AdamG} ($p = 0.5$), which employs $s(x) = 0.5x^{0.24}$. We observe that \textsc{AdamG} ($p = 0.5$) generally improves the performance on image tasks and fine-tuning GPT-2 with LoRA, but results in degraded performance when fine-tuning BERT.

However, Table~\ref{append_tab_stability} provides clearer insights into the effectiveness of the scaling $p$ through reliability comparison. Compared with the default \textsc{AdamG}, \textsc{AdamG} ($p = 0.5$) enhances adaptability on tasks that prefer a large LR Adam and reduces the adaptability on tasks that prefer a small LR Adam.
Considering the expectation-based mechanism and the empirical results, we believe that scaling $p$ may shift the range of tasks covered without narrowing that range

\section{Conclusion}
\label{sec_conclusion}
In this work, we introduced a new mechanism for making adaptive-gradient training methods parameter-free by proposing the golden step size. This step size aims to preserve tuning-free convergence and approximate the optimal step size in expectation across various optimization problems.
The resulting optimizer, \textsc{AdamG}, demonstrates improved reliability and solution quality compared to previous methods, closely matching the performance of manually tuned Adam and facilitating deployment readiness.

\textbf{Limitation.} Despite its practical success, understanding the theoretical guarantees of the proposed approach is crucial. 
We note that although the proof framework for the convergence of AdaGrad-Norm in~\cite{wang2023convergence} served as inspiration for our approach, it relies on the collective behavior of AdaGrad-Norm's step sizes throughout the entire training trajectory. 
The divergent step size derived from the progressive update formula is not yet directly connected to the AdaGrad-Norm convergence analyses in prior work~\citep{wang2023convergence, faw2023beyond, li2023convergence}; establishing this connection remains future work.
%

\section*{Acknowledgments}
This material is based in part upon work supported by the National Science Foundation under the Grant IIS-2212174, National Institute on Aging (NIA) R01AG072449, and National Institute of General
Medical Sciences (NIGMS) 1R01GM145700.

\bibliographystyle{siamplain}
\bibliography{example_references}

@article{kingma2014adam,
  title={Adam: A method for stochastic optimization},
  author={Kingma, Diederik P and Ba, Jimmy},
  journal={arXiv preprint arXiv:1412.6980},
  year={2014}
}

@inproceedings{ge2015escaping,
  title={Escaping from saddle points—online stochastic gradient for tensor decomposition},
  author={Ge, Rong and Huang, Furong and Jin, Chi and Yuan, Yang},
  booktitle={Conference on learning theory},
  pages={797--842},
  year={2015},
  organization={PMLR}
}

@article{orabona2017training,
  title={Training deep networks without learning rates through coin betting},
  author={Orabona, Francesco and Tommasi, Tatiana},
  journal={Advances in Neural Information Processing Systems},
  volume={30},
  year={2017}
}

@article{orabona2016coin,
  title={Coin betting and parameter-free online learning},
  author={Orabona, Francesco and P{\'a}l, D{\'a}vid},
  journal={Advances in Neural Information Processing Systems},
  volume={29},
  year={2016}
}

@inproceedings{chen2022better,
  title={Better parameter-free stochastic optimization with ODE updates for coin-betting},
  author={Chen, Keyi and Langford, John and Orabona, Francesco},
  booktitle={Proceedings of the AAAI Conference on Artificial Intelligence},
  volume={36},
  pages={6239--6247},
  year={2022}
}

@article{mcmahan2010adaptive,
  title={Adaptive bound optimization for online convex optimization},
  author={McMahan, H Brendan and Streeter, Matthew},
  journal={arXiv preprint arXiv:1002.4908},
  year={2010}
}

@article{faw2023beyond,
  title={Beyond uniform smoothness: A stopped analysis of adaptive sgd},
  author={Faw, Matthew and Rout, Litu and Caramanis, Constantine and Shakkottai, Sanjay},
  journal={arXiv preprint arXiv:2302.06570},
  year={2023}
}

@article{paul2021artificial,
  title={Artificial intelligence in drug discovery and development},
  author={Paul, Debleena and Sanap, Gaurav and Shenoy, Snehal and Kalyane, Dnyaneshwar and Kalia, Kiran and Tekade, Rakesh K},
  journal={Drug discovery today},
  volume={26},
  number={1},
  pages={80},
  year={2021},
  publisher={Elsevier}
}

@inproceedings{redmon2016you,
  title={You only look once: Unified, real-time object detection},
  author={Redmon, Joseph and Divvala, Santosh and Girshick, Ross and Farhadi, Ali},
  booktitle={Proceedings of the IEEE conference on computer vision and pattern recognition},
  pages={779--788},
  year={2016}
}

@article{voulodimos2018deep,
  title={Deep learning for computer vision: A brief review},
  author={Voulodimos, Athanasios and Doulamis, Nikolaos and Doulamis, Anastasios and Protopapadakis, Eftychios and others},
  journal={Computational intelligence and neuroscience},
  volume={2018},
  year={2018},
  publisher={Hindawi}
}

@article{zhuang2020adabelief,
  title={Adabelief optimizer: Adapting stepsizes by the belief in observed gradients},
  author={Zhuang, Juntang and Tang, Tommy and Ding, Yifan and Tatikonda, Sekhar C and Dvornek, Nicha and Papademetris, Xenophon and Duncan, James},
  journal={Advances in neural information processing systems},
  volume={33},
  pages={18795--18806},
  year={2020}
}

@inproceedings{balles2018dissecting,
  title={Dissecting adam: The sign, magnitude and variance of stochastic gradients},
  author={Balles, Lukas and Hennig, Philipp},
  booktitle={International Conference on Machine Learning},
  pages={404--413},
  year={2018},
  organization={PMLR}
}

@inproceedings{loizou2021stochastic,
  title={Stochastic polyak step-size for sgd: An adaptive learning rate for fast convergence},
  author={Loizou, Nicolas and Vaswani, Sharan and Laradji, Issam Hadj and Lacoste-Julien, Simon},
  booktitle={International Conference on Artificial Intelligence and Statistics},
  pages={1306--1314},
  year={2021},
  organization={PMLR}
}

@book{polyak1987introduction,
	author = {Polyak, B.T.},
	title = {Introduction to Optimization},
	publisher = {Optimization Software},
	address = {New York},
	year = {1987}
}

@article{li2023convergence,
  title={Convergence of Adam Under Relaxed Assumptions},
  author={Li, Haochuan and Jadbabaie, Ali and Rakhlin, Alexander},
  journal={arXiv preprint arXiv:2304.13972},
  year={2023}
}

@book{nesterov2018lectures,
  title={Lectures on convex optimization},
  author={Nesterov, Yurii and others},
  volume={137},
  year={2018},
  publisher={Springer}
}

@article{bubeck2015convex,
  title={Convex optimization: Algorithms and complexity},
  author={Bubeck, S{\'e}bastien and others},
  journal={Foundations and Trends{\textregistered} in Machine Learning},
  volume={8},
  number={3-4},
  pages={231--357},
  year={2015},
  publisher={Now Publishers, Inc.}
}

@article{latafat2023adaptive,
  title={Adaptive proximal algorithms for convex optimization under local Lipschitz continuity of the gradient},
  author={Latafat, Puya and Themelis, Andreas and Stella, Lorenzo and Patrinos, Panagiotis},
  journal={arXiv preprint arXiv:2301.04431},
  year={2023}
}

@article{malitsky2019adaptive,
  title={Adaptive gradient descent without descent},
  author={Malitsky, Yura and Mishchenko, Konstantin},
  journal={arXiv preprint arXiv:1910.09529},
  year={2019}
}

@article{feurer2019hyperparameter,
  title={Hyperparameter optimization},
  author={Feurer, Matthias and Hutter, Frank},
  journal={Automated machine learning: Methods, systems, challenges},
  pages={3--33},
  year={2019},
  publisher={Springer International Publishing}
}

@inproceedings{he2016deep,
  title={Deep residual learning for image recognition},
  author={He, Kaiming and Zhang, Xiangyu and Ren, Shaoqing and Sun, Jian},
  booktitle={Proceedings of the IEEE conference on computer vision and pattern recognition},
  pages={770--778},
  year={2016}
}

@inproceedings{huang2017densely,
  title={Densely connected convolutional networks},
  author={Huang, Gao and Liu, Zhuang and Van Der Maaten, Laurens and Weinberger, Kilian Q},
  booktitle={Proceedings of the IEEE conference on computer vision and pattern recognition},
  pages={4700--4708},
  year={2017}
}

@Techreport{krizhevsky2009learning,
  author = {Krizhevsky, Alex and Hinton, Geoffrey},
 address = {Toronto, Ontario},
 institution = {University of Toronto},
 publisher = {Technical report, University of Toronto},
 title = {Learning multiple layers of features from tiny images},
 year = {2009},
 title_with_no_special_chars = {Learning multiple layers of features from tiny images}
}

@article{wang2018glue,
  title={GLUE: A multi-task benchmark and analysis platform for natural language understanding},
  author={Wang, Alex and Singh, Amanpreet and Michael, Julian and Hill, Felix and Levy, Omer and Bowman, Samuel R},
  journal={arXiv preprint arXiv:1804.07461},
  year={2018}
}

@article{hu2021lora,
  title={Lora: Low-rank adaptation of large language models},
  author={Hu, Edward J and Shen, Yelong and Wallis, Phillip and Allen-Zhu, Zeyuan and Li, Yuanzhi and Wang, Shean and Wang, Lu and Chen, Weizhu},
  journal={arXiv preprint arXiv:2106.09685},
  year={2021}
}

@article{devlin2018bert,
  title={Bert: Pre-training of deep bidirectional transformers for language understanding},
  author={Devlin, Jacob and Chang, Ming-Wei and Lee, Kenton and Toutanova, Kristina},
  journal={arXiv preprint arXiv:1810.04805},
  year={2018}
}

@article{radford2019language,
  title={Language models are unsupervised multitask learners},
  author={Radford, Alec and Wu, Jeffrey and Child, Rewon and Luan, David and Amodei, Dario and Sutskever, Ilya and others},
  journal={OpenAI blog},
  volume={1},
  number={8},
  pages={9},
  year={2019}
}

@article{dosovitskiy2020vit,
  title={An Image is Worth 16x16 Words: Transformers for Image Recognition at Scale},
  author={Dosovitskiy, Alexey and Beyer, Lucas and Kolesnikov, Alexander and Weissenborn, Dirk and Zhai, Xiaohua and Unterthiner, Thomas and  Dehghani, Mostafa and Minderer, Matthias and Heigold, Georg and Gelly, Sylvain and Uszkoreit, Jakob and Houlsby, Neil},
  journal={ICLR},
  year={2021}
}

@article{simonyan2014very,
  title={Very deep convolutional networks for large-scale image recognition},
  author={Simonyan, Karen and Zisserman, Andrew},
  journal={arXiv preprint arXiv:1409.1556},
  year={2014}
}

@article{russakovsky2015imagenet,
  title={Imagenet large scale visual recognition challenge},
  author={Russakovsky, Olga and Deng, Jia and Su, Hao and Krause, Jonathan and Satheesh, Sanjeev and Ma, Sean and Huang, Zhiheng and Karpathy, Andrej and Khosla, Aditya and Bernstein, Michael and others},
  journal={International journal of computer vision},
  volume={115},
  pages={211--252},
  year={2015},
  publisher={Springer}
}

@article{ward2020adagrad,
  title={Adagrad stepsizes: Sharp convergence over nonconvex landscapes},
  author={Ward, Rachel and Wu, Xiaoxia and Bottou, Leon},
  journal={The Journal of Machine Learning Research},
  volume={21},
  number={1},
  pages={9047--9076},
  year={2020},
  publisher={JMLRORG}
}

@inproceedings{carmon2022making,
  title={Making sgd parameter-free},
  author={Carmon, Yair and Hinder, Oliver},
  booktitle={Conference on Learning Theory},
  pages={2360--2389},
  year={2022},
  organization={PMLR}
}

@article{nesterov2015universal,
  title={Universal gradient methods for convex optimization problems},
  author={Nesterov, Yu},
  journal={Mathematical Programming},
  volume={152},
  number={1-2},
  pages={381--404},
  year={2015},
  publisher={Springer}
}

@article{liu2023sophia,
  title={Sophia: A Scalable Stochastic Second-order Optimizer for Language Model Pre-training},
  author={Liu, Hong and Li, Zhiyuan and Hall, David and Liang, Percy and Ma, Tengyu},
  journal={arXiv preprint arXiv:2305.14342},
  year={2023}
}

@article{wilson2017marginal,
  title={The marginal value of adaptive gradient methods in machine learning},
  author={Wilson, Ashia C and Roelofs, Rebecca and Stern, Mitchell and Srebro, Nati and Recht, Benjamin},
  journal={Advances in neural information processing systems},
  volume={30},
  year={2017}
}

@inproceedings{sutskever2013importance,
  title={On the importance of initialization and momentum in deep learning},
  author={Sutskever, Ilya and Martens, James and Dahl, George and Hinton, Geoffrey},
  booktitle={International conference on machine learning},
  pages={1139--1147},
  year={2013},
  organization={PMLR}
}

@article{duchi2011adaptive,
  title={Adaptive subgradient methods for online learning and stochastic optimization.},
  author={Duchi, John and Hazan, Elad and Singer, Yoram},
  journal={Journal of machine learning research},
  volume={12},
  number={7},
  year={2011}
}

@inproceedings{wang2023convergence,
  title={Convergence of adagrad for non-convex objectives: Simple proofs and relaxed assumptions},
  author={Wang, Bohan and Zhang, Huishuai and Ma, Zhiming and Chen, Wei},
  booktitle={The Thirty Sixth Annual Conference on Learning Theory},
  pages={161--190},
  year={2023},
  organization={PMLR}
}

@article{ivgi2023dog,
  title={{D}o{G} is {SGD}'s Best Friend: A Parameter-Free Dynamic Step Size Schedule}, 
  author={Maor Ivgi and Oliver Hinder and Yair Carmon}, 
  journal={arXiv:2302.12022}, 
  year={2023},
}

@article{yang2024two,
  title={Two sides of one coin: the limits of untuned sgd and the power of adaptive methods},
  author={Yang, Junchi and Li, Xiang and Fatkhullin, Ilyas and He, Niao},
  journal={Advances in Neural Information Processing Systems},
  volume={36},
  year={2024}
}

@article{mishchenko2023prodigy,
  title={Prodigy: An Expeditiously Adaptive Parameter-Free Learner},
  author={Mishchenko, Konstantin and Defazio, Aaron},
  journal={arXiv preprint arXiv:2306.06101},
  year={2023}
}

@article{levy2017online,
  title={Online to offline conversions, universality and adaptive minibatch sizes},
  author={Levy, Kfir},
  journal={Advances in Neural Information Processing Systems},
  volume={30},
  year={2017}
}

@article{grimmer2019convergence,
  title={Convergence rates for deterministic and stochastic subgradient methods without Lipschitz continuity},
  author={Grimmer, Benjamin},
  journal={SIAM Journal on Optimization},
  volume={29},
  number={2},
  pages={1350--1365},
  year={2019},
  publisher={SIAM}
}

@article{khaled2023dowg,
  title={DoWG Unleashed: An Efficient Universal Parameter-Free Gradient Descent Method},
  author={Khaled, Ahmed and Mishchenko, Konstantin and Jin, Chi},
  journal={arXiv preprint arXiv:2305.16284},
  year={2023}
}

@article{defazio2023learning,
  title={Learning-rate-free learning by D-adaptation},
  author={Defazio, Aaron and Mishchenko, Konstantin},
  journal={arXiv preprint arXiv:2301.07733},
  year={2023}
}

\onecolumn

\section{Technical Appendix}

\subsection{Technical Proof}
\label{append_section_tech}

\begin{corollary}[a simple variant of Theorem 2 in~\cite{wang2023convergence}]
\label{app_corollary_convergence}
Let Assumptions~\ref{ass_lsmooth} and~\ref{ass_l0l1} hold. Then the following bound holds in expectation:
\begin{align*}
\min_{k\in [K]} ||\nabla f(\rvx_{k})||^{2} &\leq \frac{1}{\mathcal{O}(\sqrt{v_{K}})}\Big(\frac{4}{\eta}(f(\rvx_{1}) - f^{\star}) + 2D_{1}\xi(0) \\
&\quad + (2(L\eta D_{1})^{2} + D_{1}(L\eta)^{2} + \frac{1}{2}D_{0})\frac{4}{\sqrt{v_{0}}} + 2L\eta\ln v_{K}\Big).
\end{align*}
\end{corollary}

\begin{proof}
We begin by following the inequality derived from Theorem 2 of~\cite{wang2023convergence}.
\begin{align}
\frac{1}{4}\eta\sum_{k=1}^{K}\mathbb{E}[{\frac{||\nabla f(\rvx_{k})||^{2}}{\sqrt{v_{k-1}}}}] &\leq f(\rvx_{1}) - \mathbb{E}[f(\rvx_{K})] + \frac{\eta D_{1}}{2}\mathbb{E}[\xi(0) - \xi(T)] \nonumber\\
&\quad + (2\eta(L\eta D_{1})^{2} + \eta D_{1}(L\eta)^{2} + \frac{\eta}{2}D_{0})\frac{1}{\sqrt{v_{0}}} + \frac{L}{2}\eta^{2}(\mathbb{E}\ln v_{K} - \mathbb{E}\ln v_{0}).
\end{align}
Using the following result from~\cite{wang2023convergence}:
\begin{align*}
\sum_{k=1}^{K}\mathbb{E}\left[{\frac{||\nabla f(\rvx_{k})||^{2}}{\sqrt{v_{k-1}}}}\right] &\geq \mathbb{E}\left[\frac{\sum_{k=1}^{K}||\nabla f(\rvx_{k})||^{2}}{\sqrt{v_{K}}}\right]\geq \frac{\mathbb{E}\left[\sqrt{\sum_{k=1}^{K}||\nabla f(\rvx_{k})||^{2}}\right]^{2}}{\mathbb{E}[\sqrt{v_{K}}]},
\end{align*}
we have
\begin{align*}
\mathbb{E}\left[\sqrt{\sum_{k=1}^{K}||\nabla f(\rvx_{k})||^{2}}\right]^{2} &\leq \frac{4\mathbb{E}[\sqrt{v_{K}}]}{\eta}(f(\rvx_{1}) - \mathbb{E}[f(\rvx_{K})]) + 2D_{1}\mathbb{E}[\sqrt{v_{K}}]\mathbb{E}[\xi(0) - \xi(T)] \\
&\quad + (2(L\eta D_{1})^{2} + D_{1}(L\eta)^{2} + \frac{1}{2}D_{0})\frac{4\mathbb{E}[\sqrt{v_{K}}]}{\sqrt{v_{0}}} + 2L\mathbb{E}[\sqrt{v_{K}}]\eta(\mathbb{E}\ln v_{K} - \mathbb{E}\ln v_{0}).
\end{align*}
Further, using the fact that $\mathbb{E}[\frac{\sqrt{v_{K}}}{K}]$ is upper bounded by $\mathcal{O}(\frac{1}{\sqrt{K}})$~\cite{wang2023convergence}, we have in expectation that:
\begin{align*}
\min_{k\in [K]} ||\nabla f(\rvx_{k})||^{2} \leq \frac{1}{\mathcal{O}(\sqrt{K})}\Big(\cdot\Big) \leq &\frac{1}{\mathcal{O}(\sqrt{v_{K}})}\Big(\frac{4}{\eta}(f(\rvx_{1}) - f^{\star}) \\
&\qquad + 2D_{1}\xi(0) + (2(L\eta D_{1})^{2} + D_{1}(L\eta)^{2} + \frac{1}{2}D_{0})\frac{4}{\sqrt{v_{0}}} + 2L\eta\ln v_{K}\Big).
\end{align*}
This concludes the proof.
\end{proof}

\subsection{Discussion}
\label{append_subsec_dynamic_eta}

\textbf{Dynamic step size.} In case 2, $\eta$ is constant for a fixed $K$ but varies with $K$, so $\eta:=(v_{K})^{q}$ can be considered as a constant value that obeys Corollary~\ref{corollary_convergence} and the analysis in Section~\ref{subsec_gold_ada}. 
Finally, we approximate $v_{K}$ with $v_{k}$ in the proposed Algorithm~\ref{alg_god}. The approximation is intuitive and can be improved along with training. Besides, a similar idea is the optimality gap approximation, $\max_{i\leq K}||x_{i} - x_{0}||\xrightarrow[]{\text{Approx.}}||\rvx_{0} - \rvx^{\star}||$, exploited in the baseline method DoG and its variants~\cite{ivgi2023dog}.

We further discuss dynamic $\eta$ $\textit{with respect to $k$ (lower-case)}$, i.e., $\eta:= (v_{k})^{q}$, which is naturally compatible with our algorithm design. We demonstrate that it aligns well our analysis framework with some mild assumptions and eliminates the need for approximation.

In particular, we can update our Corollary 3.3 as a time-varying learning rate $\eta_{k}$, then, we have a similar form of conclusion, formulated as
\begin{align*}
\min_{k\in[K]} ||\nabla f(x_{k})||^{2} \leq &\frac{1}{\mathcal{O}(\sqrt{v_{K}})}(c_{1}\frac{\sum_{k=1}^{K}\eta_{k}^{3}}{\sum_{k=1}^{K}\eta_{k}}  + c_{2}\frac{\sum_{k=1}^{K}\eta_{k}^{2}}{\sum_{k=1}^{K}\eta_{k}}\ln v_{K}  + c_{3}\frac{1}{\sum_{k=1}^{K}\eta_{k}}),
\end{align*}
where $c_{1}$, $c_{2}$, and $c_{3}$ are the corresponding coefficients. Ignoring the constant value and substituting $(v_{k})^{q}$ for $\eta_{k}$, the right-hand side can be further reformulated as
\begin{align*}
\frac{1}{\sqrt{v_{K}}}\left(\frac{\sum_{k=1}^{K}(v_{k})^{3q}}{\sum_{k=1}^{K}(v_{k})^{q}}  + \frac{\sum_{k=1}^{K}(v_{k})^{2q}}{\sum_{k=1}^{K}(v_{k})^{q}}\ln v_{K}  + \frac{1}{\sum_{k=1}^{K}(v_{k})^{q}}\right)&\implies \mathcal{O}\left(\frac{1}{\sqrt{v_{K}}}(\frac{(v_{K})^{3q+1}}{(v_{K})^{q+1}} + \frac{(v_{K})^{2q+1}}{(v_{K})^{q+1}}\ln v_{K} + \frac{1}{(v_{K})^{q+1}})\right)\\
&\implies \mathcal{O}\left(\frac{1}{\sqrt{v_{K}}}((v_{K})^{2q}+ (v_{K})^{q}\ln v_{K} + \frac{1}{(v_{K})^{q+1}})\right).
\end{align*}
Achieving a good approximation of the above first and second formulas necessitates that $\sum_{k=1}^{K} (v_{k})^{q} \approx \frac{1}{q+1}(v_{K})^{q+1}$, which imposes requirements on the sequence $\{v_{1}, \cdots, v_{K}\}$ (e.g., $\sum_{x=1}^{K}x^{\frac{1}{2}} \approx \int_{1}^{K}x^{\frac{1}{2}}\,dx = \frac{2}{3}(K^{\frac{3}{2}}-1)$) and warrants future investigations.

However, the desired result, as shown in the third formula above, will lead to the same level of expectation in Equation~\ref{eq_defination_opt} as that derived from employing a constant step size (The difference in the last term does not affect the conclusion). This suggests that our algorithm design, $\eta_{k} = (v_{k})^{q}$, is (possibly) naturally compatible with analysis in Section~\ref{subsec_gold_ada} and Section~\ref{sub_sec_golden_step_size}.

\section{Experimental Appendix}
\label{append_section_exp}

\subsection{Setup}
\label{append_section_exp_setup}
For the image tasks, we used the same batch size of 1024 and input image sizes 32$\times$32, 32$\times$32, and 64$\times$64 for CIAFR-10, CIAFR-100, and Tiny-ImageNet datasets respectively.
For the language tasks, we used the same batch size of 32 for BERT and GPT-2 tasks. In particular, we used a LoRA rank of 4 in GPT-2 tasks.
The number of training epochs is specified in the corresponding performance table.

\textbf{Computational resources.} All the experiments can be run on a single NVIDIA RTX A5000 Graphics Card (24 GB). Each image task in one setting can be completed within 5 hours, using less than 15 GB of GPU memory. Similarly, each language task in one setting can be completed within 8 hours, using less than 10 GB of GPU memory.

\subsection{Tables}
\label{append_section_exp_comp_result}

\begin{itemize}
    \item The average performance measure of all 42 tasks is summarized in Table~\ref{app_tab_cifar10}, Table~\ref{app_tab_cifar100}, Table~\ref{app_tab_imagenet}, Table~\ref{app_tab_bert}, and Table~\ref{app_tab_gpt2}.
    \item Reliability ratios under 1\% gap and 10\% gap are provided in Table~\ref{tab_stability_virous_gap}.
    \item Results of the ablation study are provided in Table~\ref{tab_cifar10_robust_decay}, Table~\ref{tab_bert_robust_decay}, Table~\ref{append_tab_cifar10}, Table~\ref{append_tab_bert}, and Table~\ref{append_tab_stability}.
\end{itemize}

\begin{table*}[!ht]
    \centering
    \small
    \begin{threeparttable}
    \setlength{\tabcolsep}{3pt}
    \begin{tabular*}{\textwidth}{@{\extracolsep{\fill}} l *{9}{c} }
    \toprule
    \multirow{3}{*}{Algorithm}  & \multicolumn{8}{c}{Test accuracy (\%) under CIFAR-10} \\
    & \multicolumn{4}{c}{Epoch 20 \& pretrained network} & \multicolumn{4}{c}{Epoch 100 \& randomly initialized network} \\
    \cmidrule(l{4pt}r{4pt}){2-5}
    \cmidrule(l{4pt}r{4pt}){6-9}
     & DenseNet121 & ResNet18 & ViT-B/16 &  VGG11 & DenseNet121 & ResNet18 & ViT-B/16 &  VGG11 \\
    \cmidrule{1-9}
      SGD-M(1e-2) &  82.5±0.2 & 88.2±0.5 & 71.4±0.6 & 85.5±0.1 & 65.2±0.2 & 71.2±0.8 & 57.9±0.2 & 73.8±1.1 \\
      SGD-M(1e-3) &  73.7±0.0 & 74.5±0.8 & 52.7±0.6 & 77.2±0.0 & 50.6±0.5 & 63.1±0.4 & 52.4±0.5 & 50.0±1.0 \\
      SGD-M(1e-4) &  45.8±0.4 & 43.1±1.5 & 26.6±0.3 & 60.2±0.5 & 39.2±0.5 & 37.9±0.6 & 32.5±0.7 & 12.7±2.2 \\
      SGD(1e-2) &  73.8±0.1 & 75.0±0.8 & 52.9±1.3 & 77.3±0.1 & 51.5±0.3 & 63.4±0.7 & 50.7±0.3 & 46.9±0.6 \\
      SGD(1e-3) &  46.7±0.2 & 43.6±1.5 & 27.2±0.6 & 60.7±0.5 & 39.1±0.6 & 38.1±0.7 & 32.5±0.6 & 12.3±1.8 \\
      SGD(1e-4) &  14.5±0.4 & 14.9±0.5 & 11.1±0.1 & 34.0±1.1 & 21.6±0.8 & 21.1±1.1 & 23.9±1.0 & 10.3±0.4 \\
    \cmidrule{1-9}
     Adam(1e-2) &  69.4±5.5 & 80.5±2.6 & 21.2±8.0& 13.1±4.5 & 79.6±1.1 & 85.3±0.7 & 25.6±2.7& 10.0±0.0\\
     Adam(1e-3) &  88.1±0.2  & 92.4±0.4 & 75.8±0.7& 84.5±0.7 & 75.4±0.5 & 84.9±0.2 & 36.4±4.3& 77.3±0.5\\
     Adam(1e-4)  &   81.2±0.1  & 85.4±0.3 & 77.3±0.4& 84.6±0.2 & 53.4±0.2 & 63.2±1.0 & 56.3±0.7 & 71.0±0.1\\
     Adam(1e-5)  &   64.3±0.1  & 72.3±0.9 & 57.9±0.6 & 77.2±0.1 & 48.2±0.4 & 58.3±0.6 & 29.3±0.3 & 61.8±0.3\\
    \noalign{\vskip 0.5ex}\hdashline\noalign{\vskip 0.5ex}
    DoG &  78.4±0.8 &  88.3±0.7 & 63.7±0.7 & 80.5±1.8 & 62.2±0.2 &  71.1±0.7 & 54.8±0.4 & 72.9±0.2\\
    DoWG &  80.4±0.4 &  86.4±0.7 & 67.1±1.2 & 80.7±1.2 & 53.9±0.5 &  65.5±0.4 & 50.8±0.9 & 52.7±3.1\\
    \textsc{GOG} & 79.5±0.1 &  85.7±0.4 & 68.6±0.5 & 82.6±1.0 & 54.7±0.4 & 66.3±0.7 & 53.7±0.7  & 66.3±0.8\\
    D-Adapt Adam&  88.2±0.1 &  91.6±0.4 & 77.3±1.1& 71.2±10.2 & 72.3±0.3 & 83.3±0.3 & 11.3±1.2 & 49.1±27.7\\
    Prodigy Adam &  87.4±0.1 &  90.9±0.5 & 79.5±0.2 & 86.1±0.2 & 64.0±0.6 & 73.7±0.1 & 21.1±8.2 & 75.5±0.6\\
    \textsc{AdamG} &  86.1±0.3 & 91.1±0.4 & 78.6±0.4 & 87.3±0.0 & 68.1±0.6 & 75.9±0.6 & 58.1±0.3 & 77.4±0.4\\
    \bottomrule
    \end{tabular*}    
    \caption{Test accuracy with CIFAR-10 under 3 different seeds. 
    }
    \label{app_tab_cifar10}
\end{threeparttable}
\end{table*}

\begin{table*}[!ht]
    \centering
    \small
    \begin{threeparttable}
    \setlength{\tabcolsep}{3pt}
    \begin{tabular*}{\textwidth}{@{\extracolsep{\fill}} l *{9}{c} }
    \toprule
    \multirow{3}{*}{Algorithm} &  \multicolumn{8}{c}{Test accuracy (\%) under CIFAR-100} \\
    & \multicolumn{4}{c}{Epoch 20 \& pretrained network} & \multicolumn{4}{c}{Epoch 100 \& randomly initialized network} \\
    \cmidrule(l{4pt}r{4pt}){2-5}
    \cmidrule(l{4pt}r{4pt}){6-9}
     & DenseNet121 & ResNet18 & ViT-B/16 &  VGG11 & DenseNet121 & ResNet18 & ViT-B/16 &  VGG11 \\
    \cmidrule{1-9}
     Adam(1e-2) & 37.3±6.3 & 45.0±1.1 & 7.3±1.1 & 1.0±0.0 & 47.2±1.5 & 52.1±0.7 & 8.4±3.7 & 1.0±0.0\\
     Adam(1e-3) & 65.2±0.2 & 72.8±0.8 & 49.7±2.5 & 53.6±1.6 & 45.0±0.4 & 57.5±0.4 & 13.1±1.7 & 13.4±17.5 \\
     Adam(1e-4) & 55.7±0.1 & 62.3±0.7 & 51.1±0.4 & 60.1±0.3 & 23.2±0.3 & 36.4±1.0 & 27.8±0.3 & 33.5±0.5 \\
     Adam(1e-5) & 20.6±0.4 & 29.0±0.5 & 13.8±0.5 & 43.5±0.2 & 20.0±0.2 & 32.1±0.5 & 8.9±0.3 & 24.3±0.2 \\
    \noalign{\vskip 0.5ex}\hdashline\noalign{\vskip 0.5ex}
    DoG &  50.6±2.5 & 69.0±2.7 & 30.7±2.7 & 56.4±0.2 & 33.6±0.3 & 47.4±0.7 & 29.2±0.3 & 31.8±1.2 \\
    DoWG &  55.7±0.1 & 65.3±3.1 & 40.4±1.3 & 56.2±0.4 & 26.7±0.4 & 38.1±0.5 & 24.9±0.3 & 1.0±0.0 \\
    \textsc{GOG} & 54.2±1.3 & 61.7±2.4 & 35.7±1.1 & 56.8±0.2 & 25.0±0.1 & 33.7±0.6 & 27.2±0.3 & 25.0±0.1\\ 
    D-Adapt Adam & 65.4±0.0 & 71.8±1.0 & 53.6±1.0 & 43.0±5.3 & 43.7±0.6 & 55.7±0.8 & 1.0±0.1 & 29.2±0.3 \\
    Prodigy Adam & 64.4±0.1 & 72.1±0.8 & 55.9±0.3 & 62.4±0.5 & 42.0±0.2 & 53.7±0.7 & 5.7±1.6 & 41.2±0.6 \\
    \textsc{AdamG} & 62.6±0.2 & 70.4±1.3 & 54.5±0.1 & 63.1±0.1 & 35.4±0.0 & 44.9±0.4 & 31.5±0.2 & 42.1±0.4 \\
    \bottomrule
    \end{tabular*}    
    \caption{Test accuracy with CIFAR-100 under 3 different seeds.}
    \label{app_tab_cifar100}
\end{threeparttable}
\end{table*}

\begin{table*}[!ht]
    \centering
    \small
    \begin{threeparttable}
    \setlength{\tabcolsep}{3pt}
    \begin{tabular*}{\textwidth}{@{\extracolsep{\fill}} l *{9}{c} }
    \toprule
    \multirow{3}{*}{Algorithm} & \multicolumn{8}{c}{Test accuracy (\%) under Tiny-ImageNet} \\
    & \multicolumn{4}{c}{Epoch 20 \& pretrained network} & \multicolumn{4}{c}{Epoch 100 \& randomly initialized network} \\
    \cmidrule(l{4pt}r{4pt}){2-5}
    \cmidrule(l{4pt}r{4pt}){6-9}
     & DenseNet121 & ResNet18 & ViT-B/16 &  VGG11 & DenseNet121 & ResNet18 & ViT-B/16 &  VGG11 \\
    \cmidrule{1-9}
     Adam(1e-2) &  38.5±0.6 & 43.4±0.5 & 3.9±2.0 & 0.5±0.0 & 37.2±0.9 & 45.5±0.2 & 1.6±0.9 & 0.5±0.0 \\
     Adam(1e-3) &  62.9±0.2 & 63.0±0.2 & 57.3±0.4 & 12.1±16.5 & 39.2±0.1 & 50.9±0.5 & 7.4±0.9 & 16.9±11.6 \\
     Adam(1e-4) &  59.5±0.3 & 60.0±1.4 & 56.5±0.2 & 59.6±0.2 & 16.8±0.2 & 35.7±0.4 & 16.4±0.2 & 35.2±1.0 \\
     Adam(1e-5) &  35.2±0.2 & 24.8±0.8 & 20.9±0.1 & 51.3±0.2 & 16.0±0.3 & 24.9±0.3 & 10.9±0.2 & 22.0±0.4 \\
    \noalign{\vskip 0.5ex}\hdashline\noalign{\vskip 0.5ex}
    DoG &  61.4±0.3 & 69.1±2.0 & 49.5±0.9 & 57.4±1.7 & 34.5±0.3 & 45.5±0.5 & 14.2±0.3 & 24.9±1.0 \\
    DoWG &  60.7±0.3 & 61.1±3.9 & 45.3±2.1 & 57.2±0.1 & 24.4±0.1 & 28.5±0.4 & 15.5±0.1 & 7.8±4.5 \\
    \textsc{GOG} & 60.4±0.1 & 58.4±3.6 & 41.2±0.4 & 58.4±0.3 & 22.4±0.3 & 31.1±1.0 & 17.2±0.5 & 20.4±0.4 \\
    D-Adapt Adam &  60.2±0.2 & 60.3±0.8 & 64.5±0.3 & 23.1±7.4 & 36.0±0.1 & 47.3±0.4 & 1.2±0.9 & 27.4±0.5 \\
    Prodigy Adam &  62.0±0.2 & 63.6±1.0 & 63.2±0.3 & 58.8±0.0 & 40.5±0.4 & 53.6±0.4 & 8.1±1.1 & 33.8±0.1 \\
    \textsc{AdamG} &  62.7±0.1 & 64.2±1.4 & 60.2±0.3 & 59.7±0.3 & 26.3±0.4 & 39.0±0.2 & 16.9±0.1 & 35.9±0.4 \\
    \bottomrule
    \end{tabular*}    
    \caption{Test accuracy with Tiny-Imagenet under 3 different seeds.}
    \label{app_tab_imagenet}
\end{threeparttable}
\end{table*}

\begin{table*}[!ht]
    \centering
    \small
    \begin{threeparttable}
    \setlength{\tabcolsep}{1.5pt}
    \begin{tabular*}{\textwidth}{@{\extracolsep{\fill}} l *{9}{c} }
    \toprule
    \multirow{3}{*}{Algorithm} & \multicolumn{7}{c}{Fine-tuning pretrained BERT under GLUE benchmark \& Epoch 3} \\
    \cmidrule(l{4pt}r{4pt}){2-10} & SST-2 & MRPC & QQP & MNLI & QNLI &  RTE & WNLI\\
     & Acc. & F1\&Acc. & F1\&Acc. & Matched acc.\& & Acc. &  Acc. & Acc.\\
     &  &  &  & Mismatched acc. & &  & \\
    \cmidrule{1-10}
     Adam(1e-2) & 50.3±0.9 & 54.1±38.3\&56.1±17.3 &  17.9±25.4\&54.4±12.4 & 33.0±1.7\&33.0±1.6 & 49.5±0.0 & 49.1±2.6 & 52.1±6.0 \\
     Adam(1e-3) & 50.3±0.9 & 81.2±0.0\&68.4±0.0 & 0.0±0.0\&63.2±0.0 & 32.1±0.4\&32.2±0.5 & 49.8±0.5 & 50.9±2.6 & 52.1±6.0 \\
     Adam(1e-4) & 77.1±7.6 & 87.0±1.5\&81.6±1.4 & 17.9±25.4\&54.4±12.4 & 77.7±0.5\&77.8±0.2 & 85.3±0.6 & 63.9±3.5 & 47.4±12.6 \\
     Adam(1e-5) & 92.5±0.3 & 88.5±0.5\&83.2±0.9 & 87.2±0.2\&90.7±0.1 & 84.1±0.1\&84.4±0.2 & 91.3±0.3 & 65.8±1.2 & 38.0±10.2 \\
    \noalign{\vskip 0.5ex}\hdashline\noalign{\vskip 0.5ex}
    DoG & 91.4±0.3 & 83.2±1.7\&74.3±4.2 & 85.5±0.4\&89.1±0.0 & 83.1±0.2\&83.8±0.3 & 90.6±0.1 & 51.9±3.3 & 57.3±1.3 \\
    DoWG &  74.8±17.3 & 82.3±1.8\&72.3±2.6 & 55.8±39.4\&79.5±11.5 & 59.5±20.6\&60.3±21.1 & 74.6±17.8 & 51.1±2.7 & 52.1±6.0\\
    \textsc{GOG} & 91.5±0.3  & 89.7±0.4\&85.6±0.1 & 85.1±0.2\&88.9±0.1  & 82.5±0.2\&83.3±0.2  & 90.8±0.2  & 66.2±2.4  & 52.1±6.0\\
    D-Adapt Adam &76.6±18.1 & 81.2±0.0\&68.4±0.0 & 0.0±0.0\&63.2±0.0 & 66.1±24.2\&66.4±24.5 & 73.9±17.3 & 61.3±9.9 & 52.1±6.0 \\
    Prodigy Adam & 91.5±1.3 & 82.0±5.4\&73.5±7.5 & 87.3±0.1\&90.4±0.2 & 83.1±0.5\&83.6±0.6 & 90.8±0.1 & 65.8±3.5 & 46.5±13.9 \\
    \textsc{AdamG} & 90.9±0.4 & 87.0±2.1\&81.5±3.3 & 87.1±0.1\&90.4±0.0 & 83.9±0.4\&84.3±0.1 & 89.8±0.3 & 65.2±3.5 & 52.1±6.0 \\
    \bottomrule
    \end{tabular*}
    \caption{Performance of fine-tuning pretrained BERT with GLUE benchmark under 3 different seeds.}
    \label{app_tab_bert}
\end{threeparttable}
\end{table*}

\begin{table*}[!ht]
    \centering
    \small
    \begin{threeparttable}
    \setlength{\tabcolsep}{3pt}
    \begin{tabular*}{\textwidth}{@{\extracolsep{\fill}} l *{9}{c} }
    \toprule
    \multirow{3}{*}{Algorithm} & \multicolumn{7}{c}{Fine-tuning LoRA on GPT-2 under GLUE benchmark \& Epoch 3} \\
    \cmidrule(l{4pt}r{4pt}){2-10}
     & SST-2 & MRPC & QQP & MNLI & QNLI &  RTE & WNLI\\
     & Acc. & F1\&Acc. & F1\&Acc. & Matched acc.\& & Acc. &  Acc. & Acc.\\
     &  &  &  & Mismatched acc. & &  & \\
    \cmidrule{1-10}
     Adam(1e-2) & 50.3±0.9 & 70.8±14.7\&61.6±9.6 & 32.7±7.0\&65.8±0.6 & 32.3±0.4\&32.5±0.5 & 50.0±0.4 & 52.5±1.2 & 48.4±5.7 \\
     Adam(1e-3) & 88.1±0.3 & 84.5±0.4\&76.1±0.9 & 67.3±7.2\&71.1±10.0 & 75.6±0.4\&77.5±0.4 & 82.4±0.8 & 60.0±4.7 & 42.3±5.3 \\
     Adam(1e-4) & 90.8±0.2 & 81.3±1.1\&71.3±0.7 & 81.8±0.2\&86.1±0.0 & 78.8±0.1\&80.2±0.3 & 84.9±0.4 & 61.6±2.2 & 44.1±0.7 \\
     Adam(1e-5) & 88.1±0.4 & 78.1±2.4\&66.3±3.2 & 77.1±0.4\&82.0±0.1 & 72.8±0.4\&74.4±0.3 & 79.8±1.1 & 51.5±3.7 & 47.9±6.0 \\
    \noalign{\vskip 0.5ex}\hdashline\noalign{\vskip 0.5ex}
    DoG & 64.2±9.3 & 80.4±0.5\&67.9±0.2 & 43.0±19.8\&67.8±2.0 & 43.8±0.0\&45.4±0.4 & 51.7±2.0 & 50.1±1.3 & 48.8±5.4 \\
    DoWG & 90.4±0.9 & 80.7±1.3\&69.8±0.6 & 77.4±0.4\&81.5±0.3 & 72.8±0.4\&74.8±0.2 & 81.9±0.9 & 50.8±3.9 & 46.9±6.7 \\
    \textsc{GOG} & 90.0±1.0 & 52.6±7.5\&45.4±3.0  &77.2±0.6\&81.7±0.4 & 73.7±0.1\&75.6±0.3  &  81.4±0.5 & 53.1±1.5  &  52.1±6.0\\ 
    D-Adapt Adam & 55.2±4.4 & 63.5±25.0\&58.5±14.0 & 27.6±16.8\&65.0±1.3 & 32.8±0.1\&33.0±0.0 & 50.2±0.5 & 50.8±1.8 & 50.7±5.3 \\
    Prodigy Adam & 85.7±1.9 & 81.0±0.3\&68.6±0.3 & 27.5±21.1\&64.6±1.0 & 33.1±0.5\&33.2±0.3 & 52.2±1.4 & 50.9±2.8 & 51.6±5.7 \\
    \textsc{AdamG} & 90.9±0.6 & 82.6±0.6\&72.5±1.4 & 80.8±0.4\&85.6±0.1 & 78.8±0.1\&79.9±0.2 & 86.0±0.5 & 58.0±4.9 & 49.8±5.8 \\
    \bottomrule
    \end{tabular*}
    \caption{Performance of fine-tuning LoRA on GPT-2 with GLUE benchmark under 3 different seeds.}
    \label{app_tab_gpt2}
\end{threeparttable}
\end{table*}


\begin{table*}[!ht]
    \centering
    \small
    \begin{threeparttable}
    \setlength{\tabcolsep}{3pt}
    \begin{tabular*}{\textwidth}{@{\extracolsep{\fill}} l *{10}{c} }
    \toprule
    \multirow{3}{*}{Dataset} & \multirow{3}{*}{Algorithm}  & \multicolumn{8}{c}{Test accuracy (\%)} \\
    & & \multicolumn{4}{c}{Epoch 20\&pretrained network} & \multicolumn{4}{c}{Epoch 100\&randomly init. network} \\
    \cmidrule(l{4pt}r{4pt}){3-6}
    \cmidrule(l{4pt}r{4pt}){7-10}
     & & DenseNet & ResNet & ViT-B &  VGG & DenseNet & ResNet & ViT-B &  VGG \\
    \cmidrule{1-10}
     \multirow{2}{*}{CIFAR-10} & \colorbox{gray!30}{\textsc{AdamG}} &  86.1±0.3 & 91.1±0.4 & 78.6±0.4 & 87.3±0.0 & 68.1±0.6$^{\times}$ & 75.9±0.6$^{\times}$ & 58.1±0.3 & 77.4±0.4\\
    & \textsc{AdamG} ($p = 0.5$) & 87.4±0.2&92.9±0.6&79.6±0.2&87.4±0.5&74.5±0.4$^{\times}$&82.4±0.2&58.8±0.1&78.8±0.2\\
    \cmidrule{1-10}
     \multirow{2}{*}{CIFAR-100} & \colorbox{gray!30}{\textsc{AdamG}} & 62.6±0.2 & 70.4±1.3 & 54.5±0.1 & 63.1±0.1 & 35.4±0.0$^{\times}$ & 44.9±0.4$^{\times}$ & 31.5±0.2 & 42.1±0.4 \\
    & \textsc{AdamG} ($p = 0.5$) & 65.0±0.2&74.0±1.4&54.9±0.3&63.7±0.3 &43.8±0.1&53.4±0.2&32.5±0.2&42.7±1.3\\
    \cmidrule{1-10}
     \multirow{2}{*}{Tiny-ImageNet} & \colorbox{gray!30}{\textsc{AdamG}} &  62.7±0.1 & 64.2±1.4 & 60.2±0.3 & 59.7±0.3 & 26.3±0.4$^{\times}$ & 39.0±0.2$^{\times}$ & 16.9±0.1 & 35.9±0.4 \\
    & \textsc{AdamG} ($p = 0.5$) &63.6±0.2&65.4±1.3&64.1±0.5&57.7±0.0&34.8±0.1&45.4±0.3&17.9±0.3&33.7±0.6\\
    \bottomrule
    \end{tabular*}  
    \caption{Test accuracy with CIFAR-10, CIFAR-100, and Tiny-ImageNet under 3 different seeds. }
    \label{append_tab_cifar10}
\end{threeparttable}
\end{table*}

\begin{table*}[!ht]
    \centering
    \small
    \begin{threeparttable}
    \setlength{\tabcolsep}{1.5pt}
    \begin{tabular*}{\textwidth}{@{\extracolsep{\fill}} l *{10}{c} }
    \toprule
      \multirow{2}{*}{Algorithm} & CoLA & SST-2 & MRPC & STS-B & QQP & MNLI & QNLI &  RTE & WNLI\\
     \cmidrule(l{4pt}r{4pt}){2-10}
      & Matthews corr. & Acc. & F1 & Pearson corr.  & F1 & Matched Acc. & Acc. &  Acc. & Acc.\\
    \cmidrule{1-10}
    \multicolumn{10}{l}{BERT} \\
    \cmidrule{1-10}
     \colorbox{gray!30}{\textsc{AdamG}} & 50.6±3.2$^{\times}$ & 90.9±0.4 & 87.0±2.1 & 88.7±0.6 & 87.1±0.1 & 83.9±0.4 & 89.8±0.3 & 65.2±3.5 & 52.1±6.0 \\
    \textsc{AdamG} ($p = 0.5$) & 0.0±0.0$^{\times}$ & 50.9±0.0$^{\times}$ & 81.2±0.0$^{\times}$ & 26.3±29.5$^{\times}$ & 0.0±0.0$^{\times}$ & 36.6±6.1$^{\times}$ & 49.5±0.0$^{\times}$ & 50.9±2.6$^{\times}$ & 52.1±6.0 \\
    \cmidrule{1-10}
    \multicolumn{10}{l}{LoRA on GPT-2 } \\
    \cmidrule{1-10}
    \colorbox{gray!30}{\textsc{AdamG}} & 24.2±5.0 & 90.9±0.6 & 82.6±0.6 & 83.9±0.5 & 80.8±0.4 & 78.8±0.1 & 86.0±0.5 & 58.0±4.9 & 49.8±5.8 \\
    \textsc{AdamG} ($p = 0.5$) & 31.0±7.4 & 90.4±0.4 & 82.8±0.7 & 85.4±0.8 & 82.7±0.3 & 79.9±0.2 & 86.7±0.7 & 56.4±1.2 & 52.6±5.3\\
    \bottomrule
    \end{tabular*}
    \caption{Performance when fine-tuning pretrained BERT and GPT-2 with LoRA on the GLUE benchmark over three random seeds.}
    \label{append_tab_bert}
\end{threeparttable}
\end{table*}

\begin{table*}[!ht]
    \centering
    \small
    \begin{threeparttable}
    \begin{tabular*}{\textwidth}{@{\extracolsep{\fill}} l *{7}{c} }
    \toprule
     Metrics  & Algorithm & Adam(1e-2) & Adam(1e-3) & Adam(1e-4) &  Adam(1e-5) & Avg.\\
    \cmidrule{1-7}
    \multirow{2}{*}{Reliability ratio} & \colorbox{gray!30}{\textsc{AdamG}} &  2/5  &  11/14  &  15/15  &  7/8   & \textbf{0.76}\\
    & \textsc{AdamG} ($p=0.5$)  &  4/5  &  13/14  &  15/15  &  0/8 & 0.68\\
    \bottomrule
    \end{tabular*}    
    \caption{\textit{Reliability} ratio comparison, which is derived from Table~\ref{append_tab_cifar10} and Table~\ref{append_tab_bert}.}
    \label{append_tab_stability}
\end{threeparttable}
\end{table*}

\clearpage
\subsection{Figures}
\label{append_loss_img}
The training loss curves corresponding to Table~\ref{tab_cifar10} and Table~\ref{tab_bert} are demonstrated as Figure~\ref{app_fig_resnet10}, Figure~\ref{app_fig_resnet100},  Figure~\ref{app_fig_TINYIMAGENET}, and Figure~\ref{fig_lan_exp_comb}.

\begin{figure*}[!ht]
\begin{center}
\subfigure[DenseNet121\&Pretrained]{\includegraphics[width=0.2\linewidth]{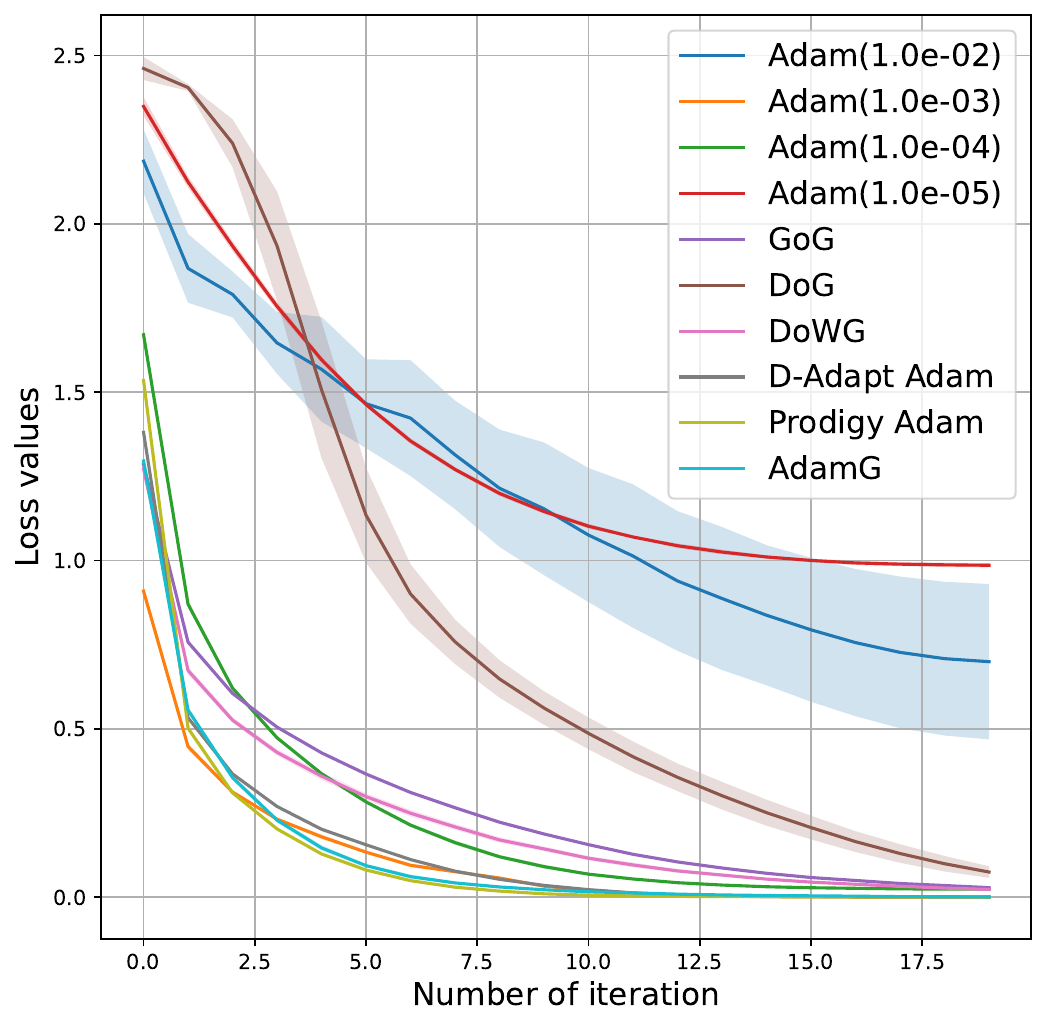}}
\subfigure[ResNet18\&Pretrained]{\includegraphics[width=0.2\linewidth]{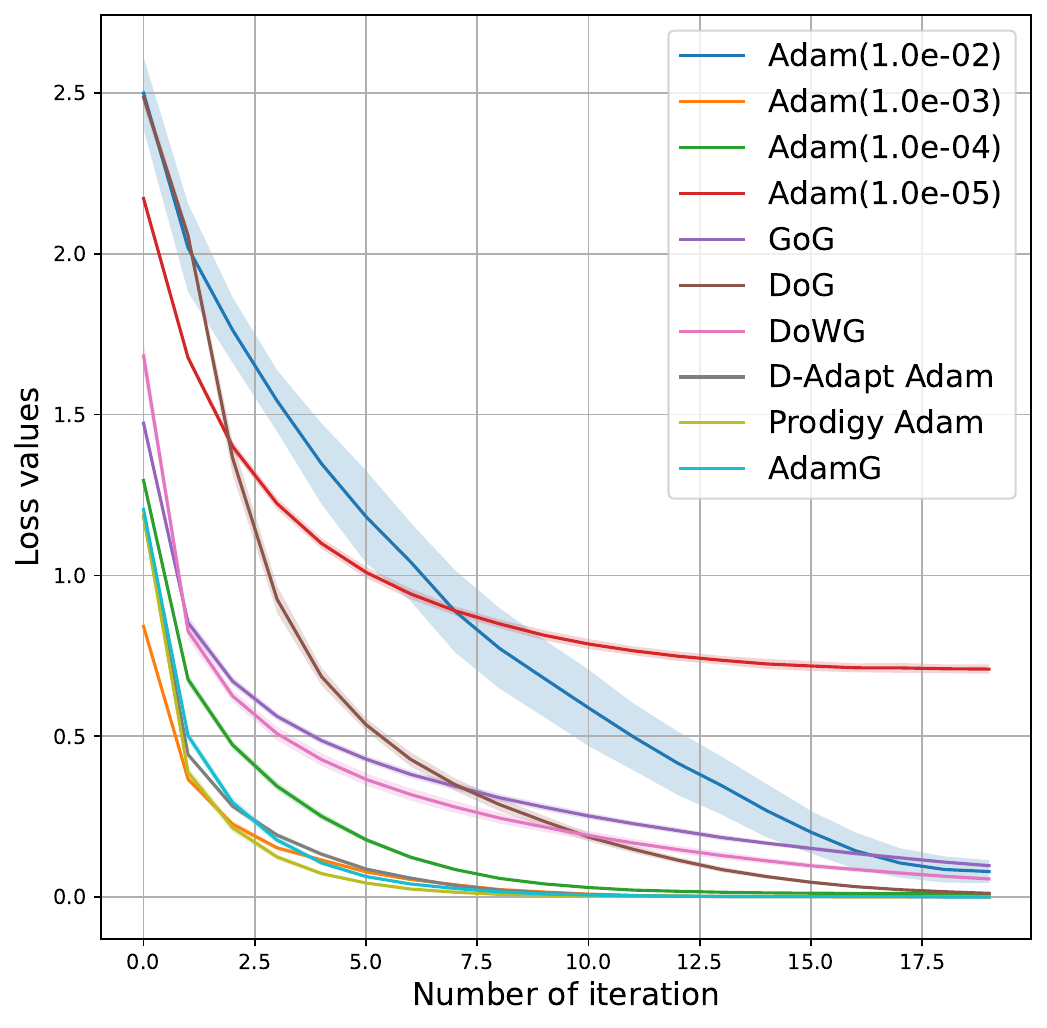}}
\subfigure[ViT-B/16\&Pretrained]{\includegraphics[width=0.2\linewidth]{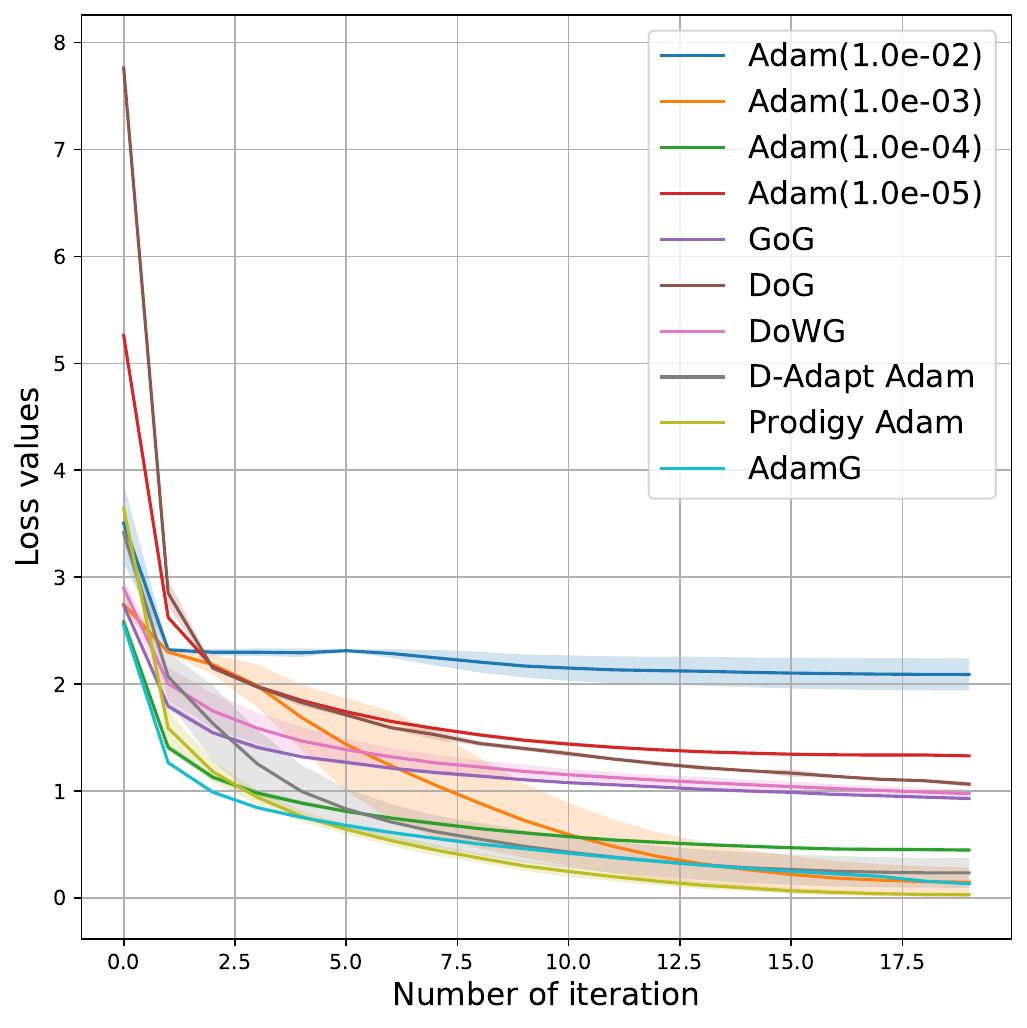}}
\subfigure[VGG11\&Pretrained]{\includegraphics[width=0.2\linewidth]{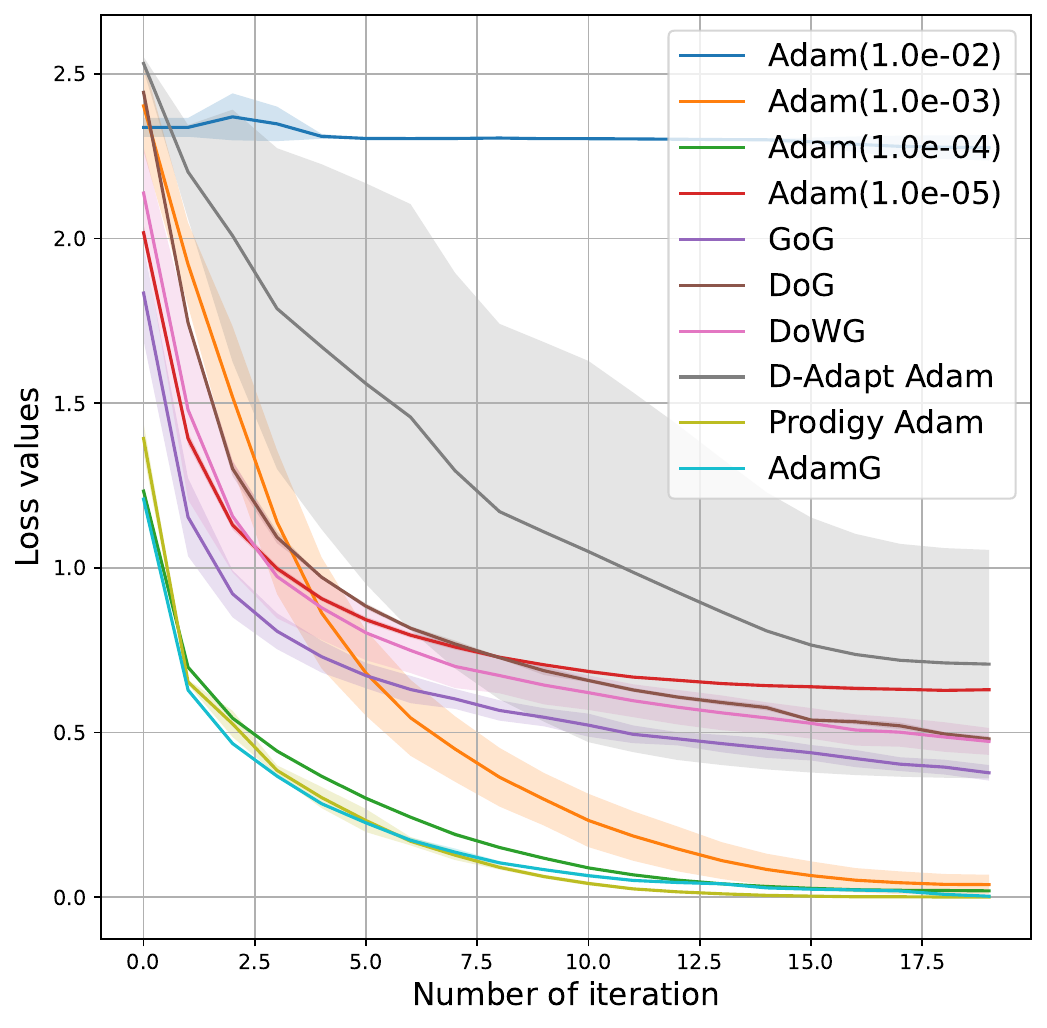}}
\subfigure[DenseNet121\&R.I.]{\includegraphics[width=0.2\linewidth]{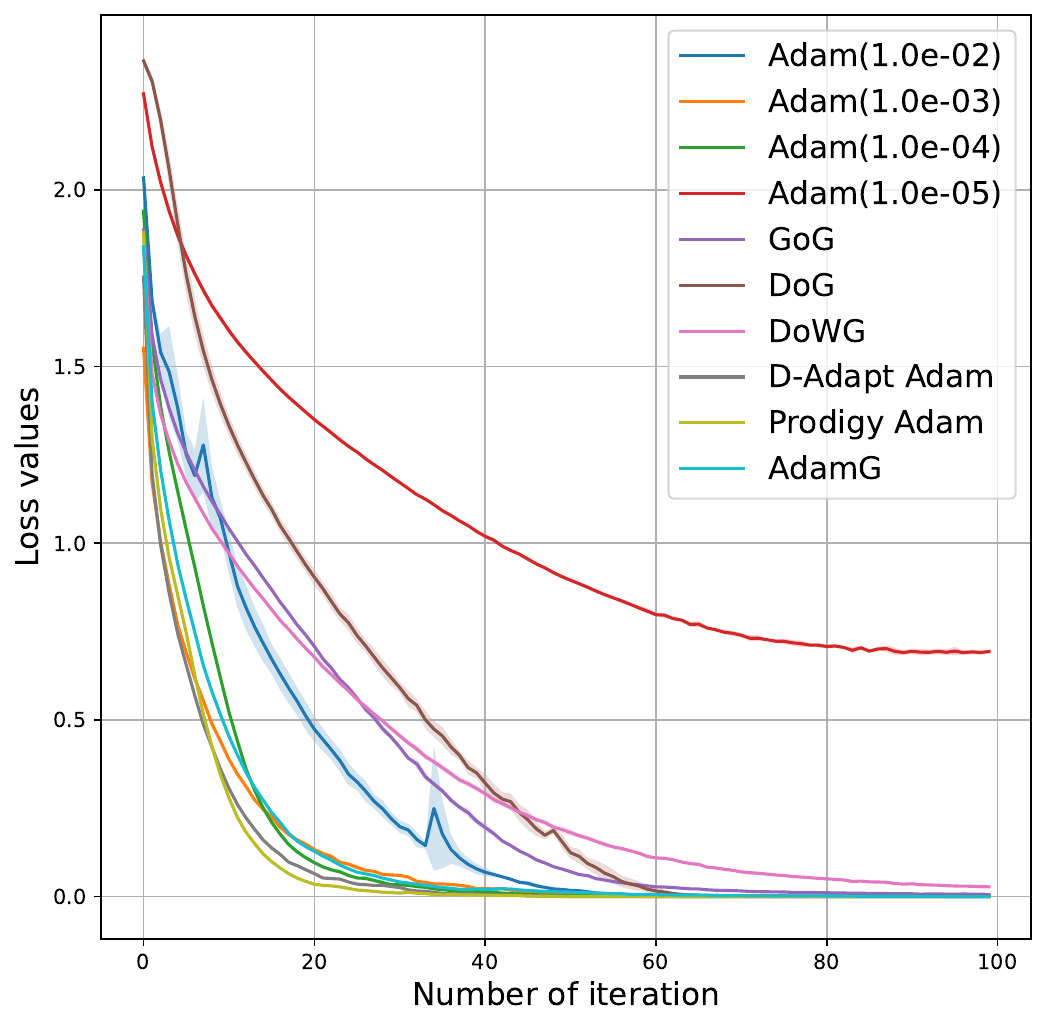}}
\subfigure[ResNet18\&R.I.]{\includegraphics[width=0.2\linewidth]{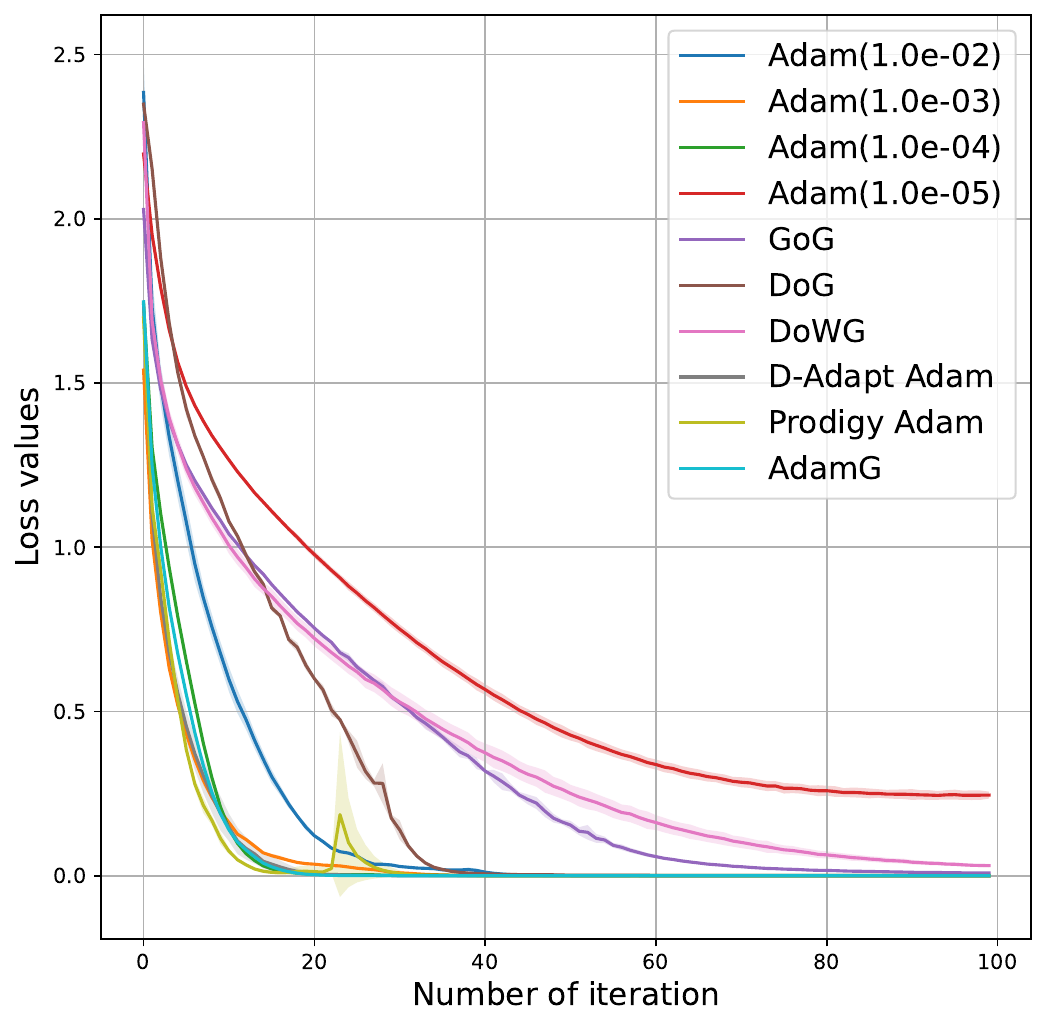}}
\subfigure[ViT-B/16\&R.I.]{\includegraphics[width=0.2\linewidth]{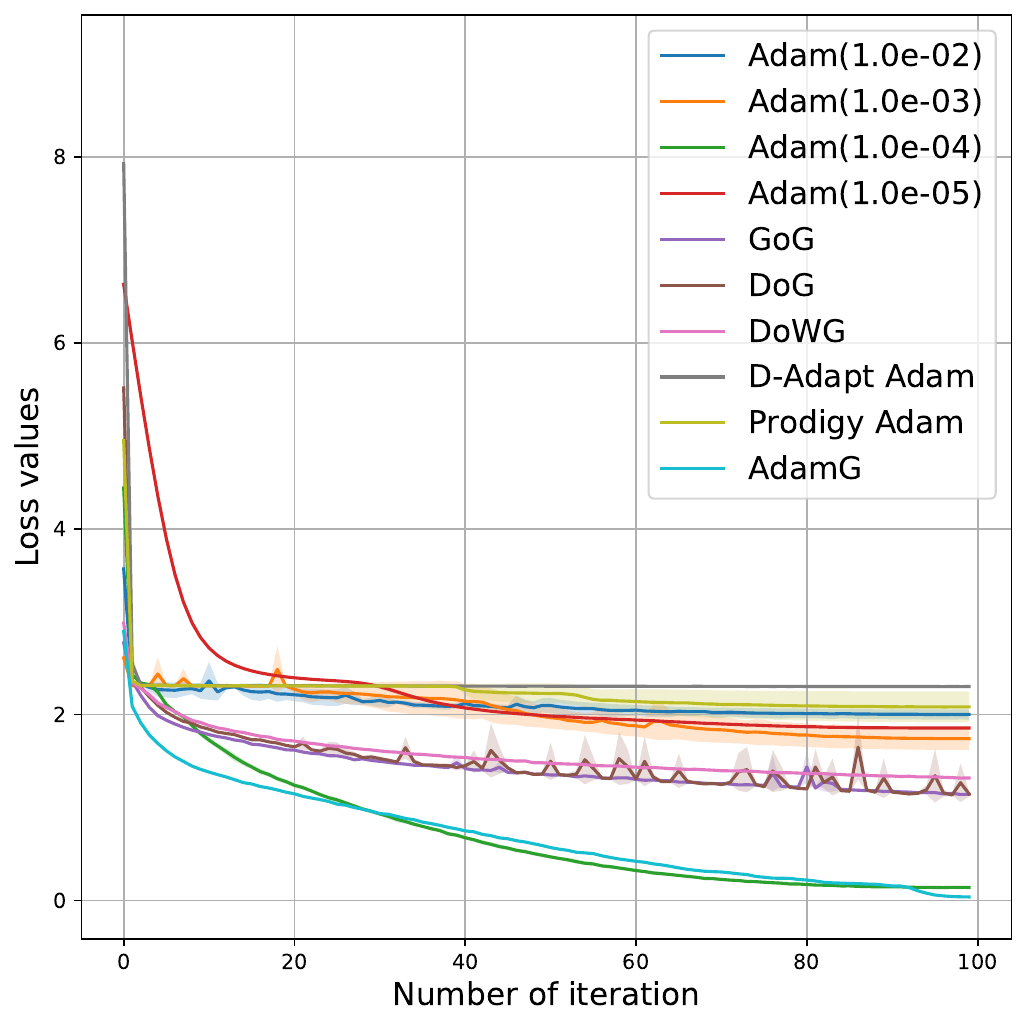}}
\subfigure[VGG11\&R.I.]{\includegraphics[width=0.2\linewidth]{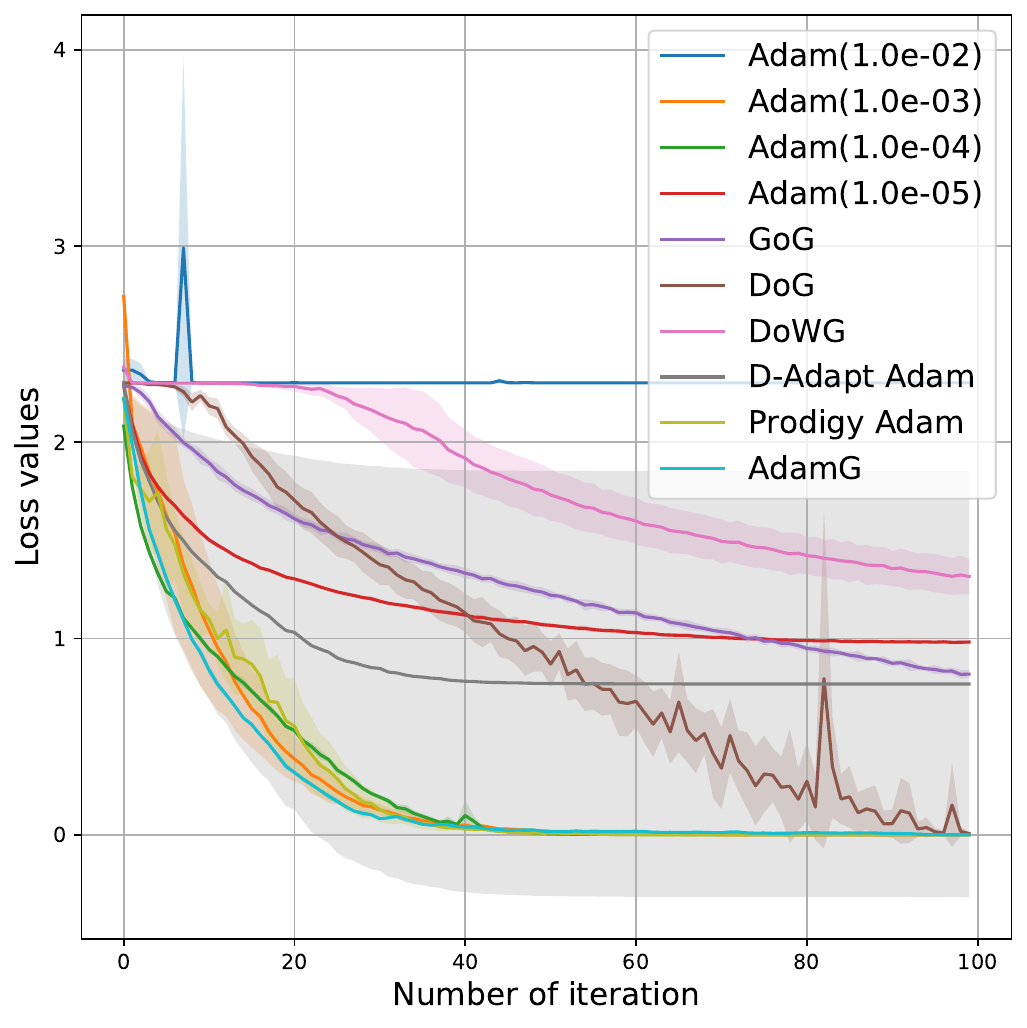}}
\end{center}
\caption{CIFAR-10 experiments. R.I. denotes randomly initialized networks.
\label{app_fig_resnet10}}
\end{figure*}

\begin{figure*}[!ht]
\begin{center}
\subfigure[DenseNet121\&Pretrained]{\includegraphics[width=0.2\linewidth]{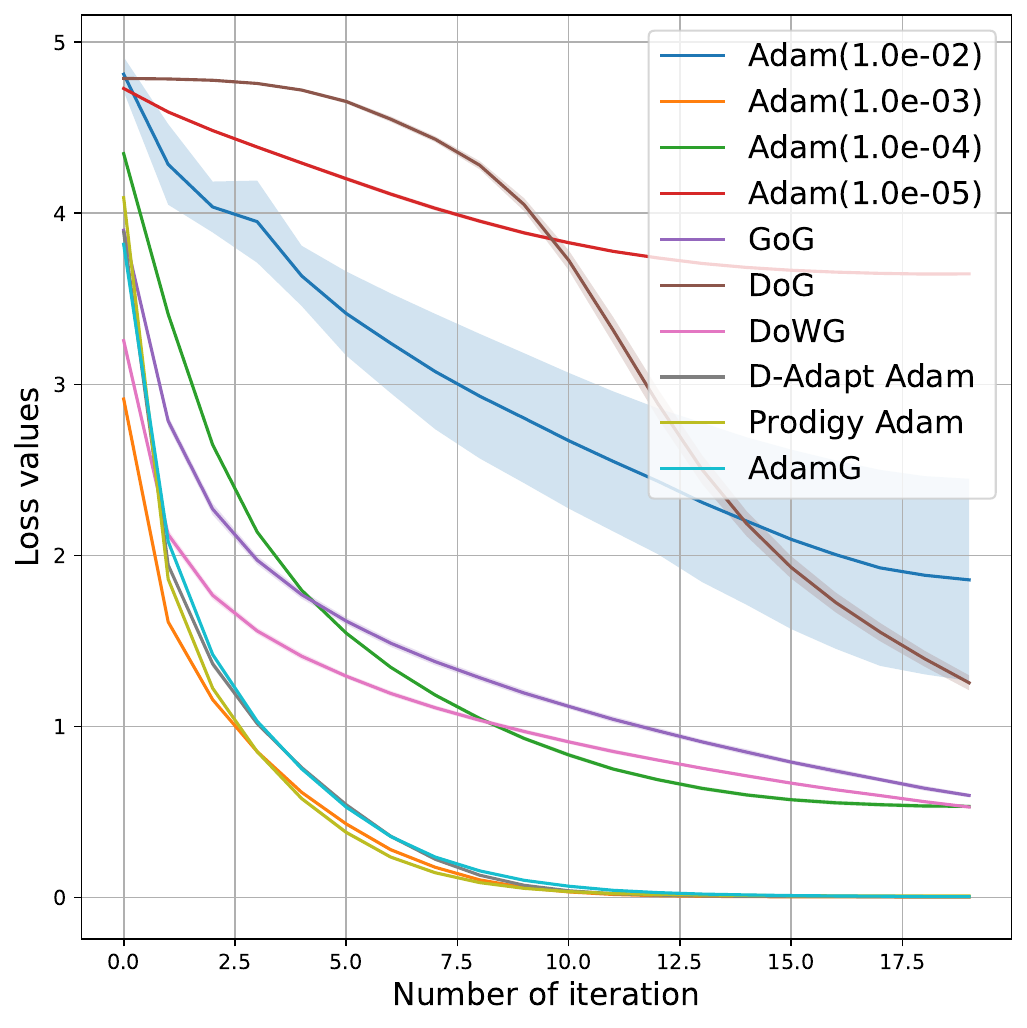}}
\subfigure[ResNet18\&Pretrained]{\includegraphics[width=0.2\linewidth]{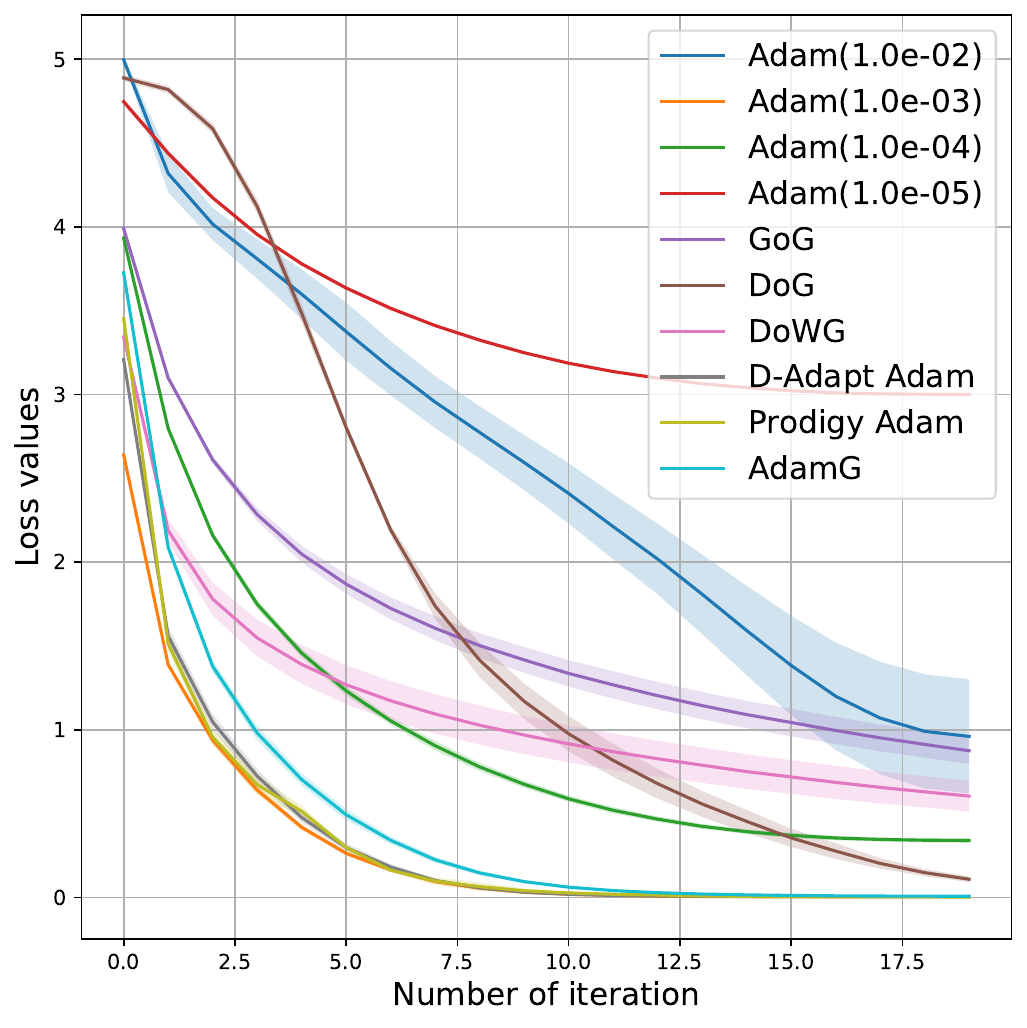}}
\subfigure[ViT-B/16\&Pretrained]{\includegraphics[width=0.2\linewidth]{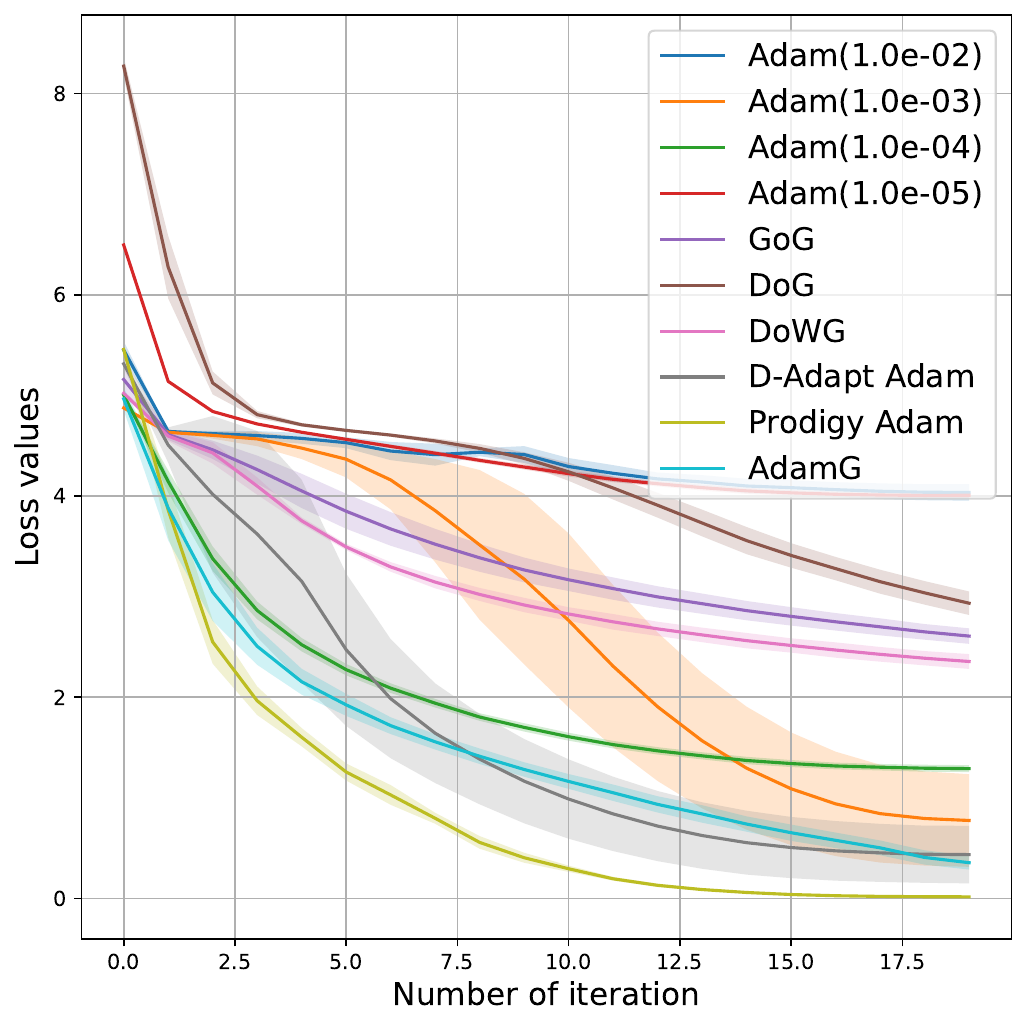}}
\subfigure[VGG11\&Pretrained]{\includegraphics[width=0.2\linewidth]{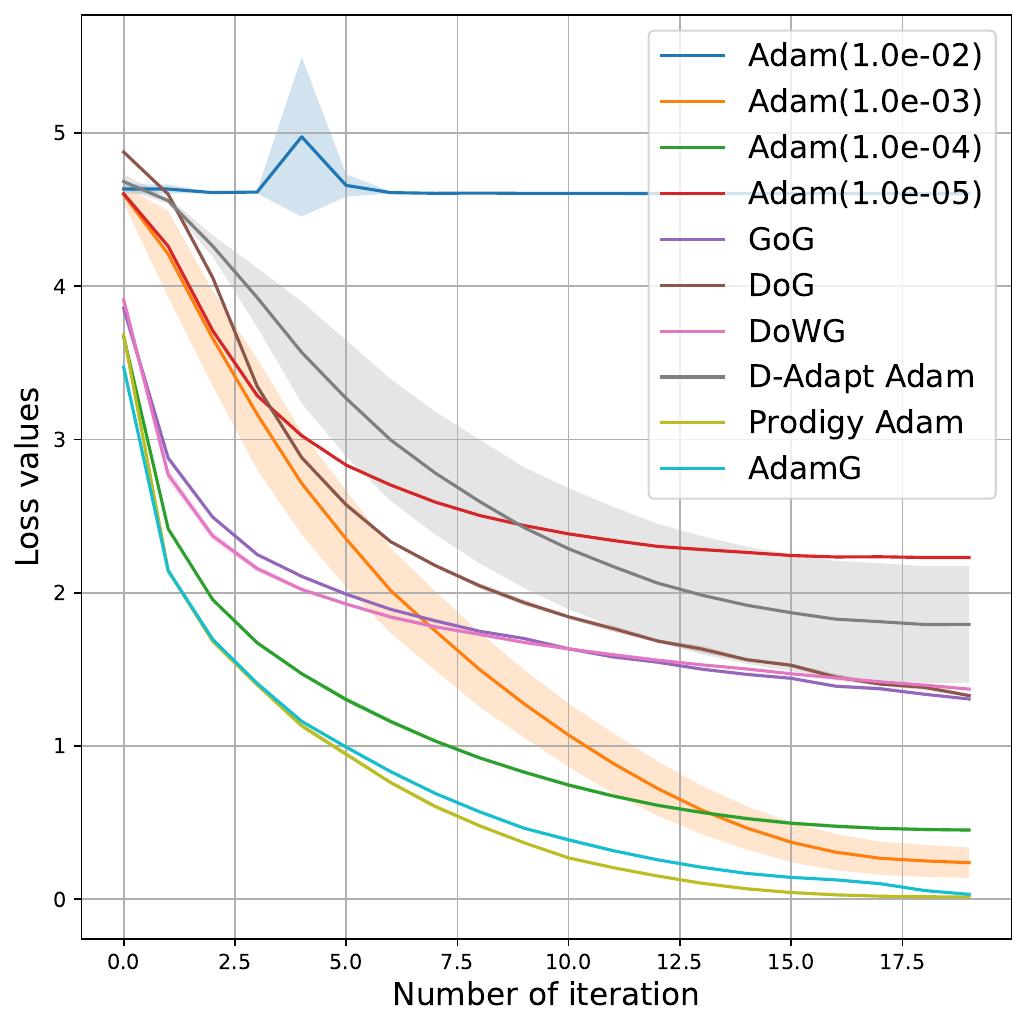}}
\subfigure[DenseNet121\&R.I.]{\includegraphics[width=0.2\linewidth]{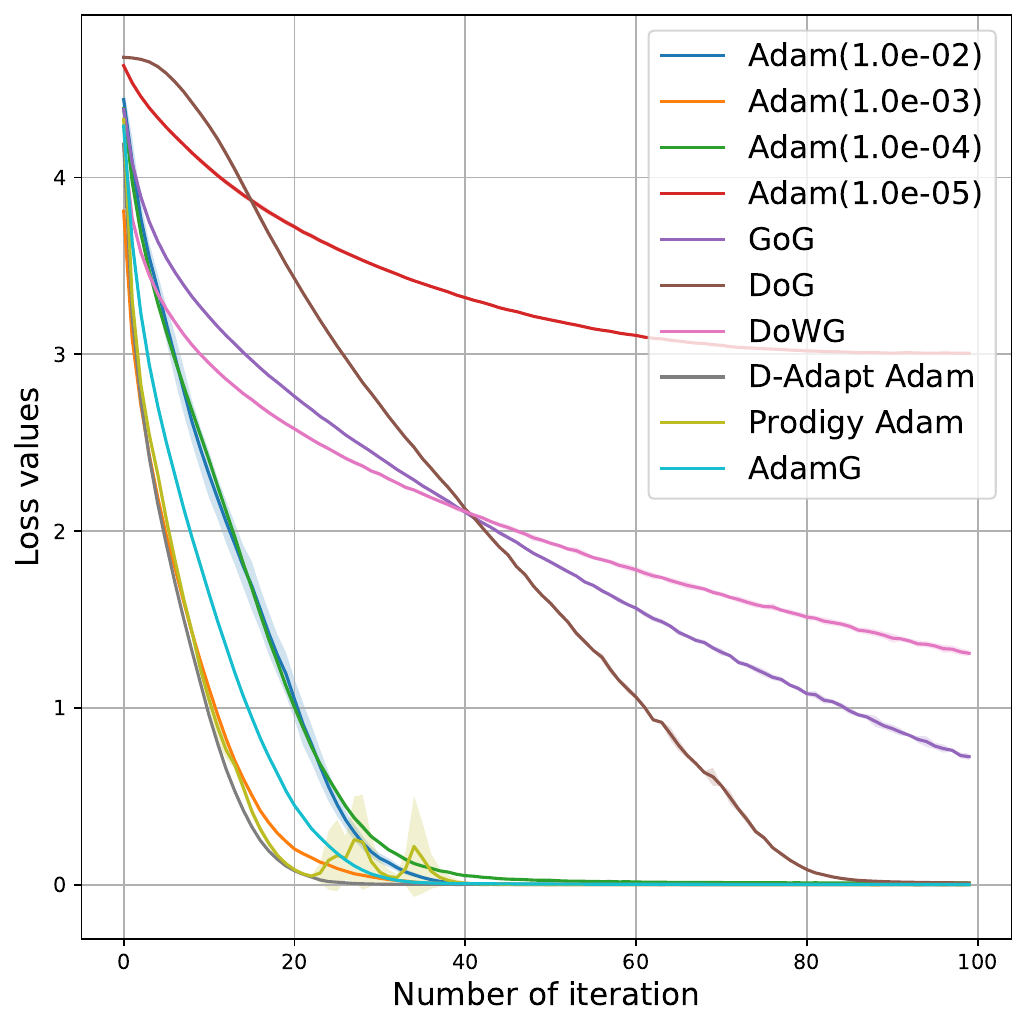}}
\subfigure[ResNet18\&R.I.]{\includegraphics[width=0.2\linewidth]{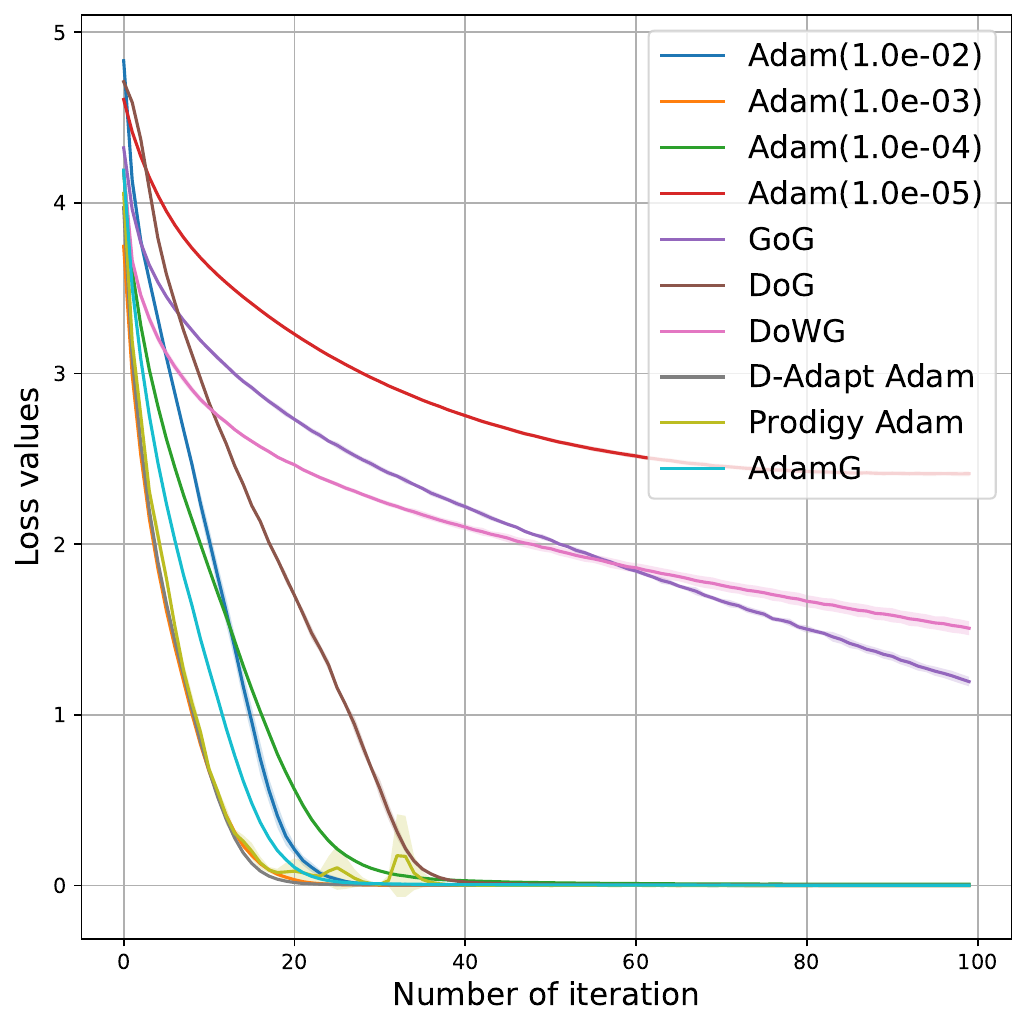}}
\subfigure[ViT-B/16\&R.I.]{\includegraphics[width=0.2\linewidth]{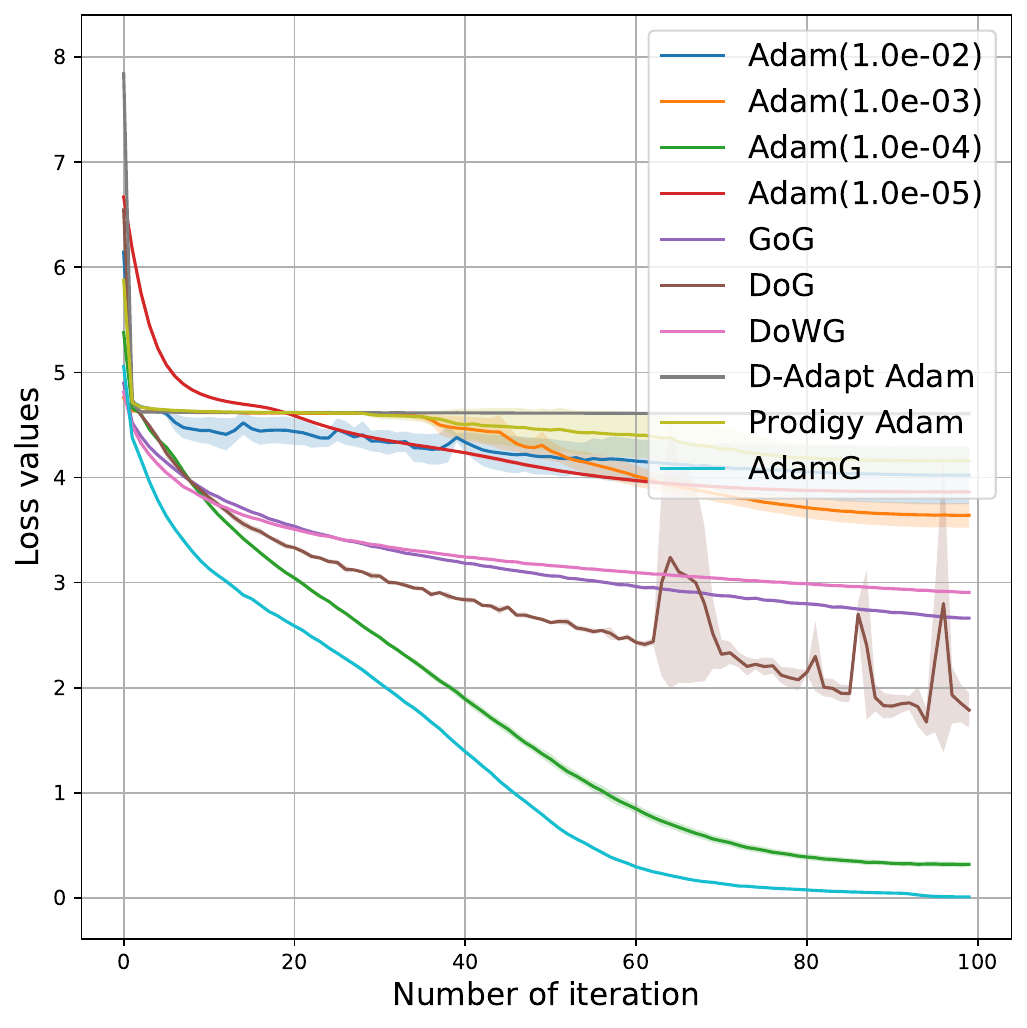}}
\subfigure[VGG11\&R.I.]{\includegraphics[width=0.2\linewidth]{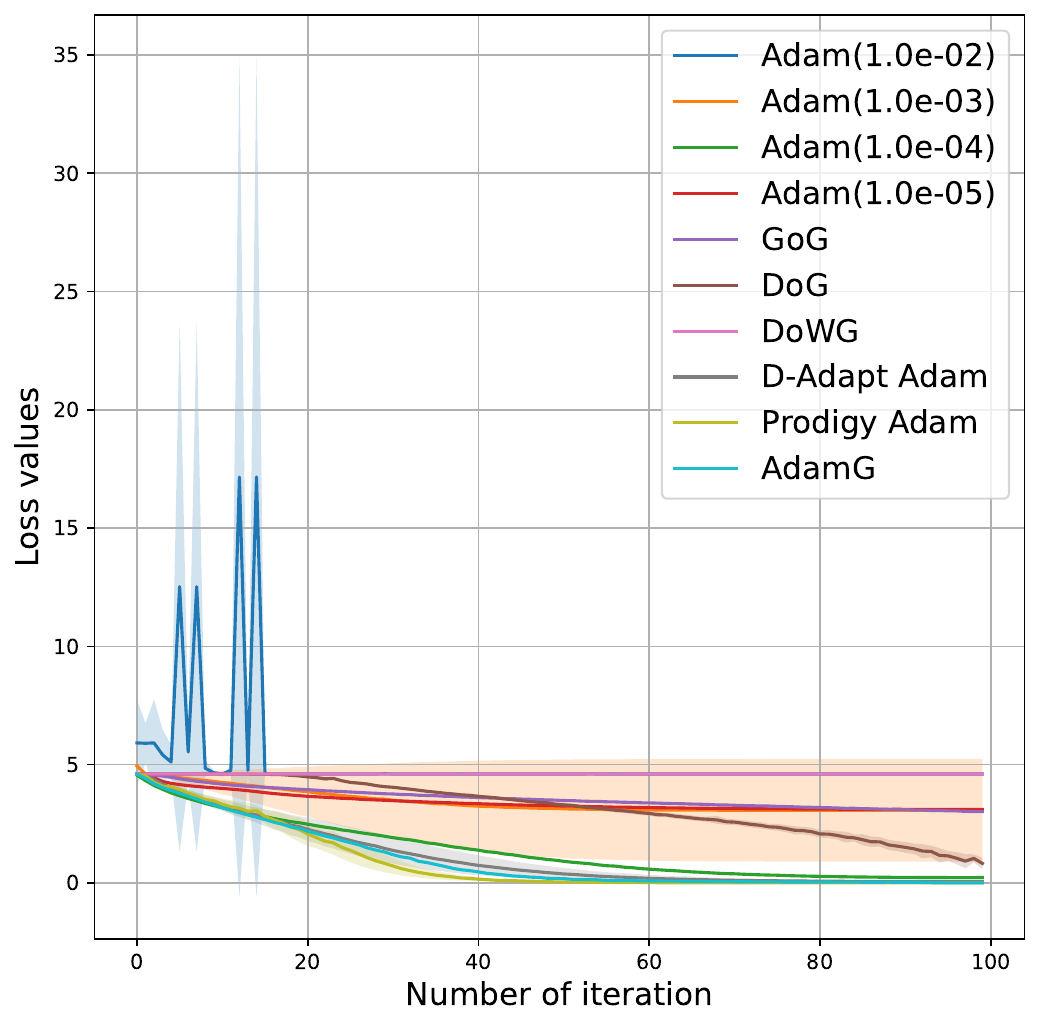}}
\end{center}
\caption{CIFAR-100 experiments. R.I. denotes randomly initialized networks.
\label{app_fig_resnet100}}
\end{figure*}

\begin{figure*}[!ht]
\begin{center}
\subfigure[DenseNet121\&Pretrained]{\includegraphics[width=0.24\linewidth]{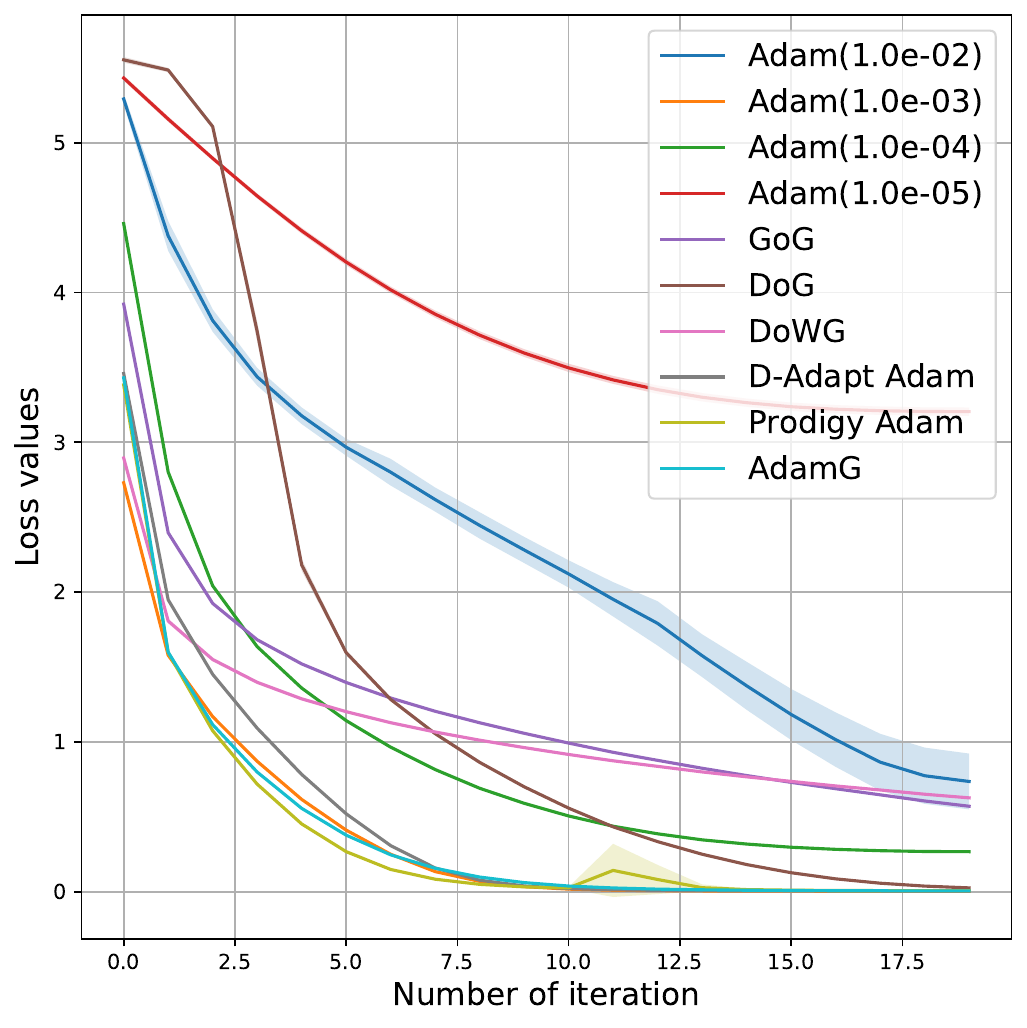}}
\subfigure[ResNet18\&Pretrained]{\includegraphics[width=0.24\linewidth]{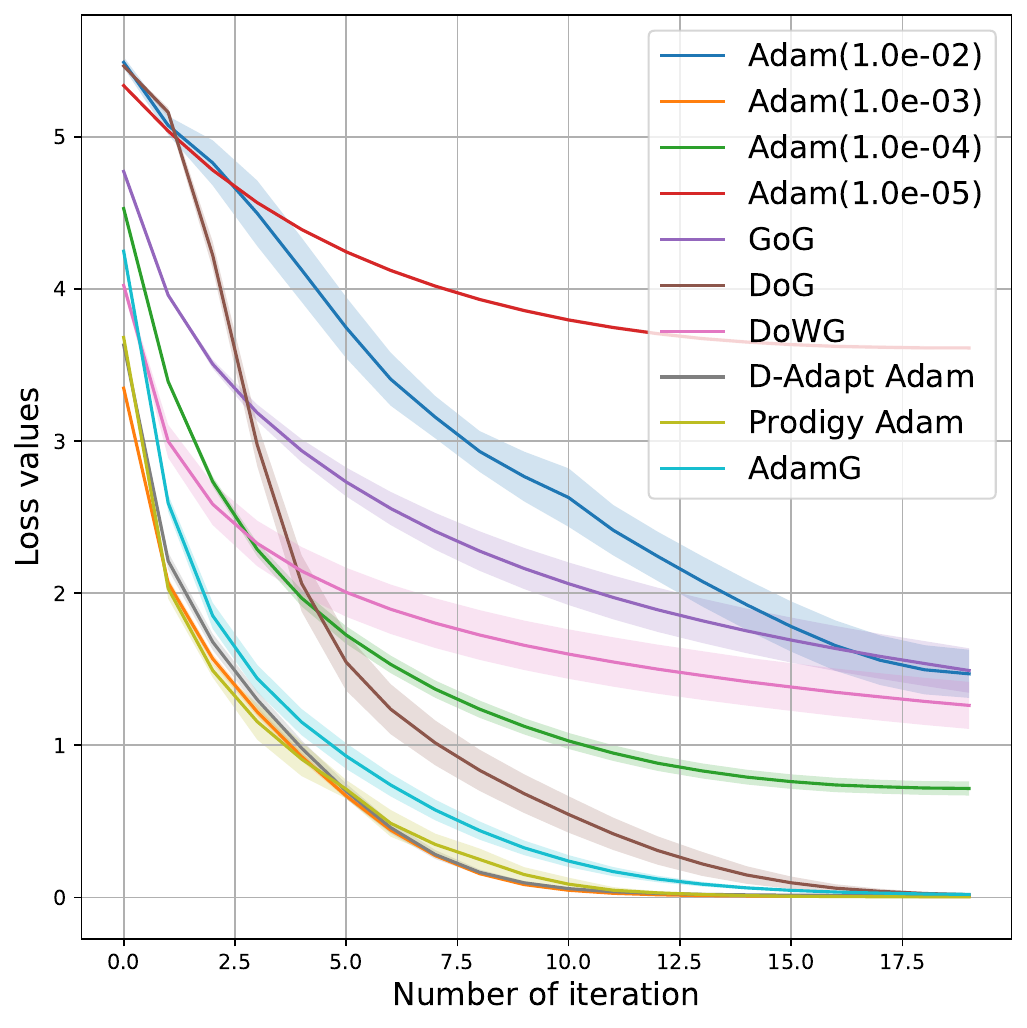}}
\subfigure[ViT-B/16\&Pretrained]{\includegraphics[width=0.24\linewidth]{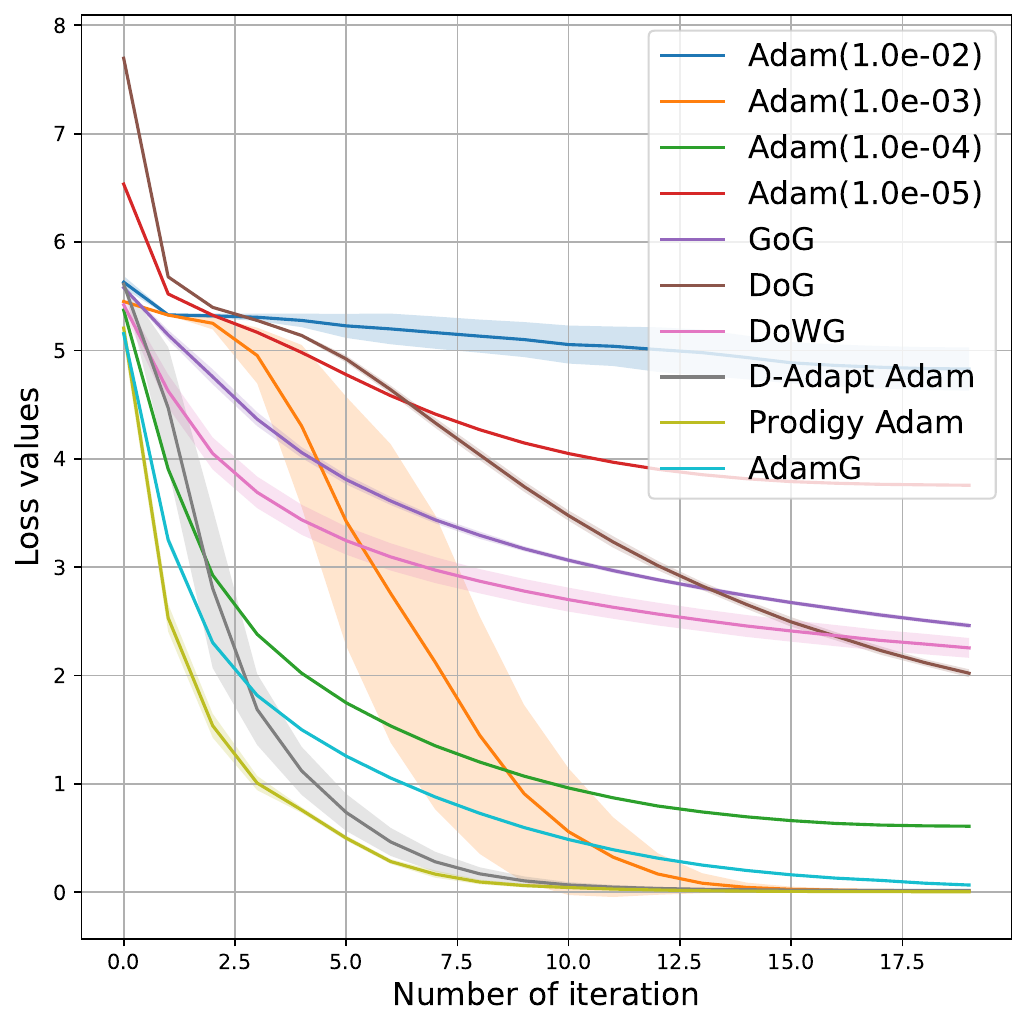}}
\subfigure[VGG11\&Pretrained]{\includegraphics[width=0.24\linewidth]{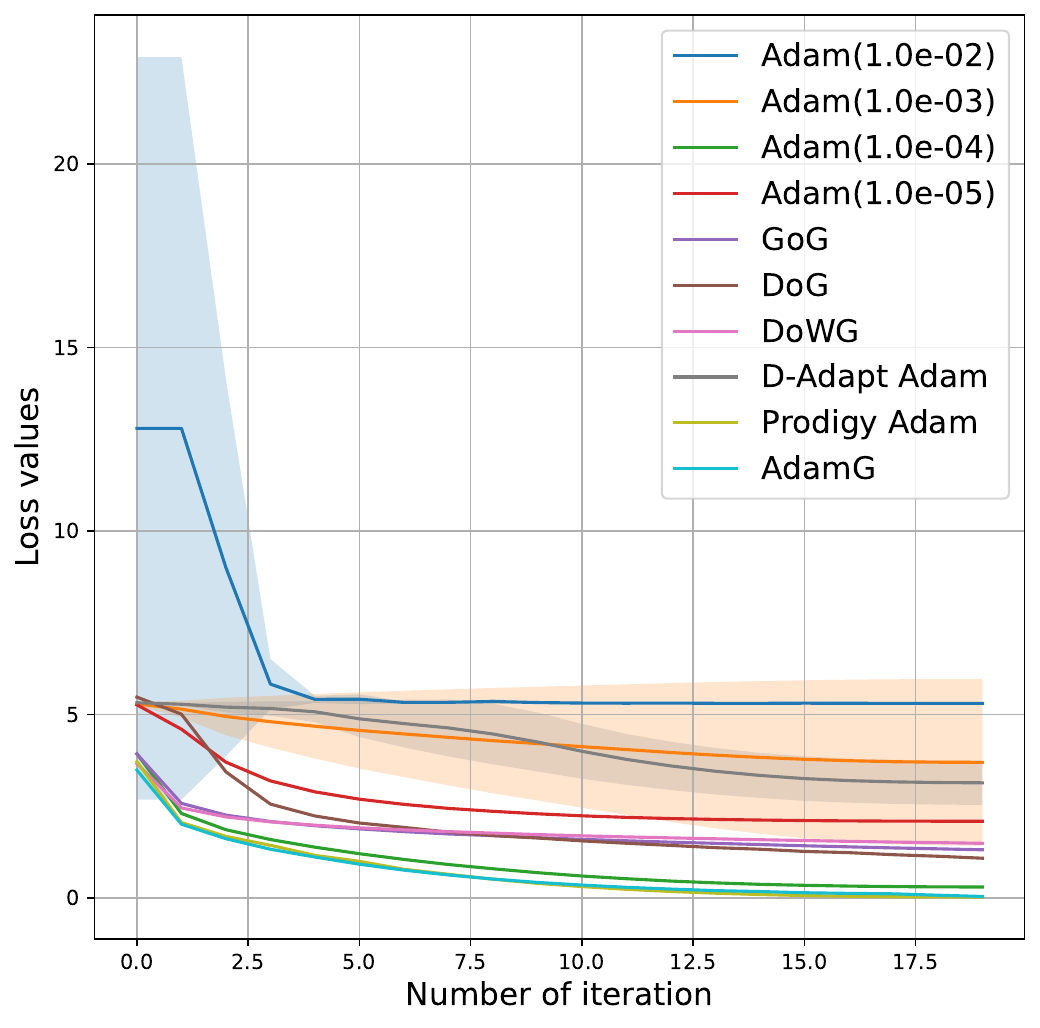}}
\subfigure[DenseNet121\&R.I.]{\includegraphics[width=0.24\linewidth]{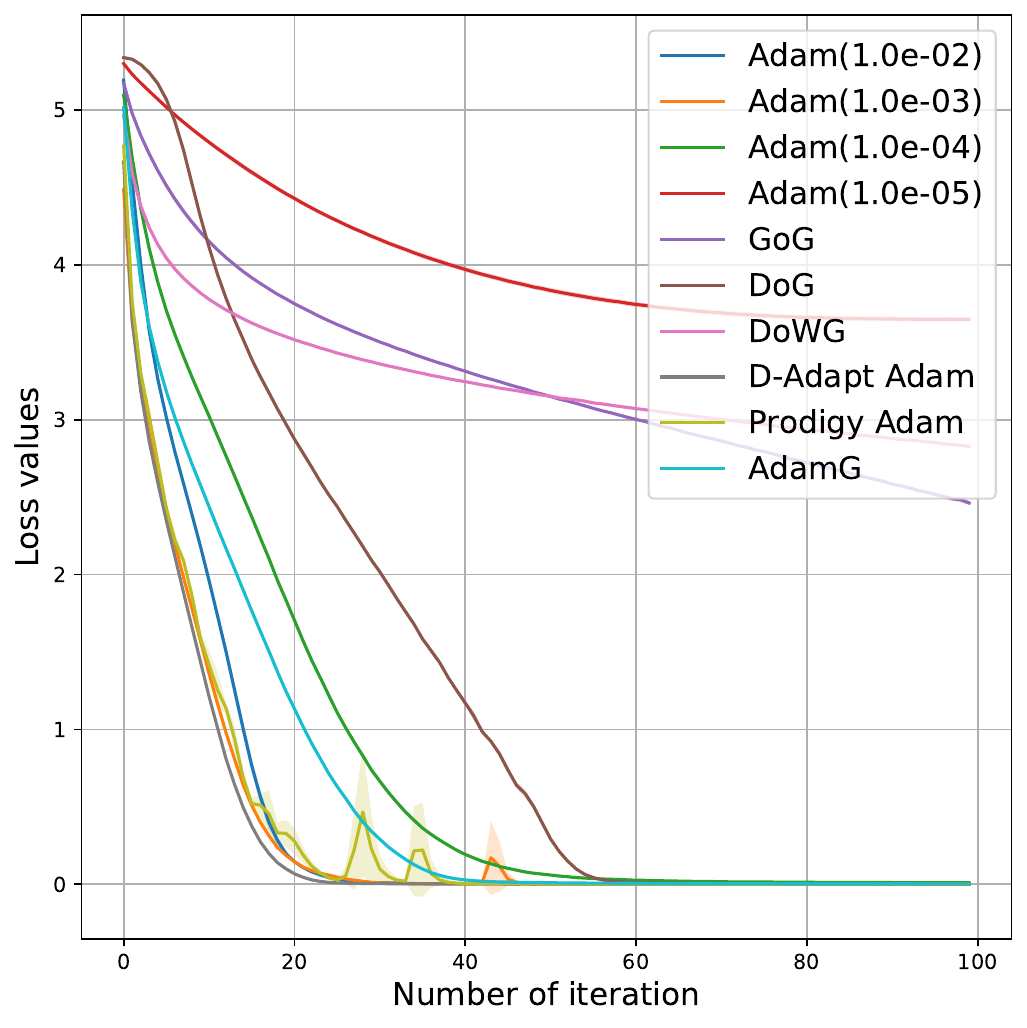}}
\subfigure[ResNet18\&R.I.]{\includegraphics[width=0.24\linewidth]{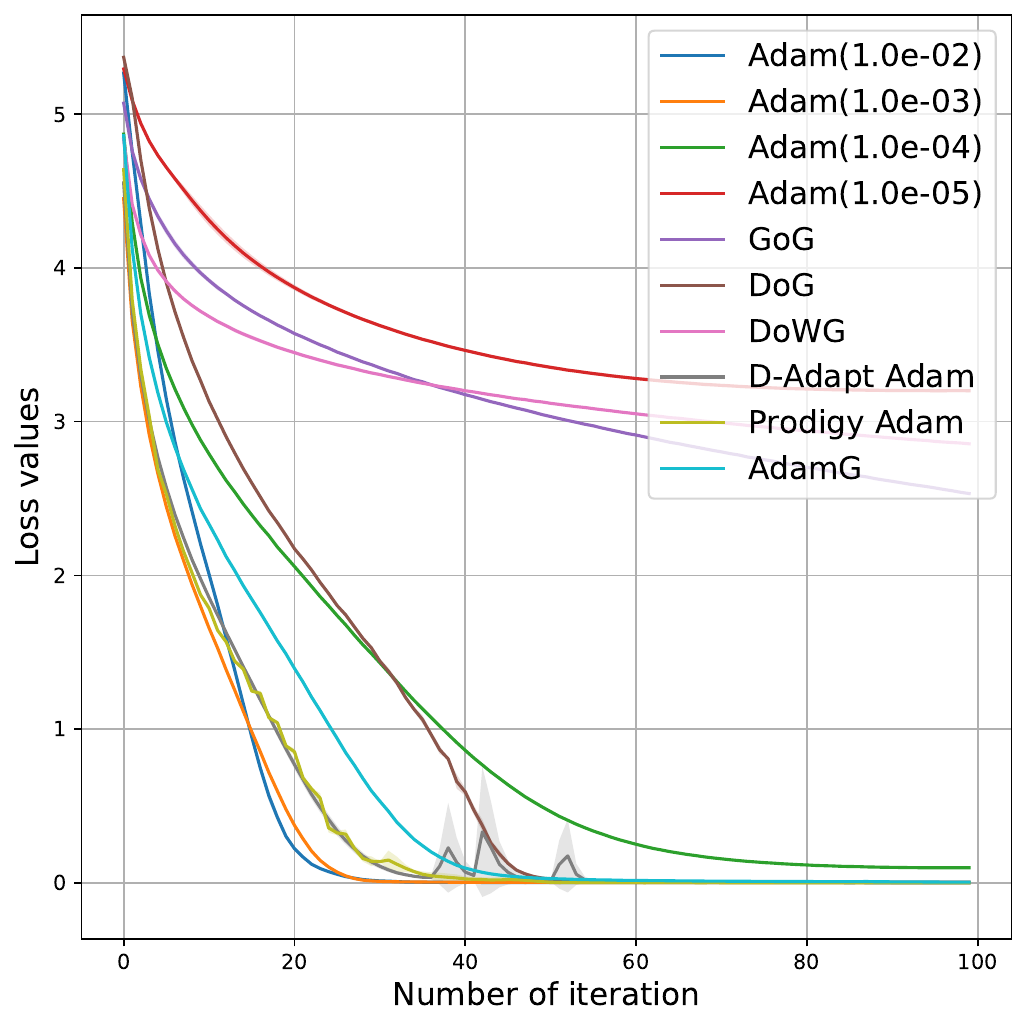}}
\subfigure[ViT-B/16\&R.I.]{\includegraphics[width=0.24\linewidth]{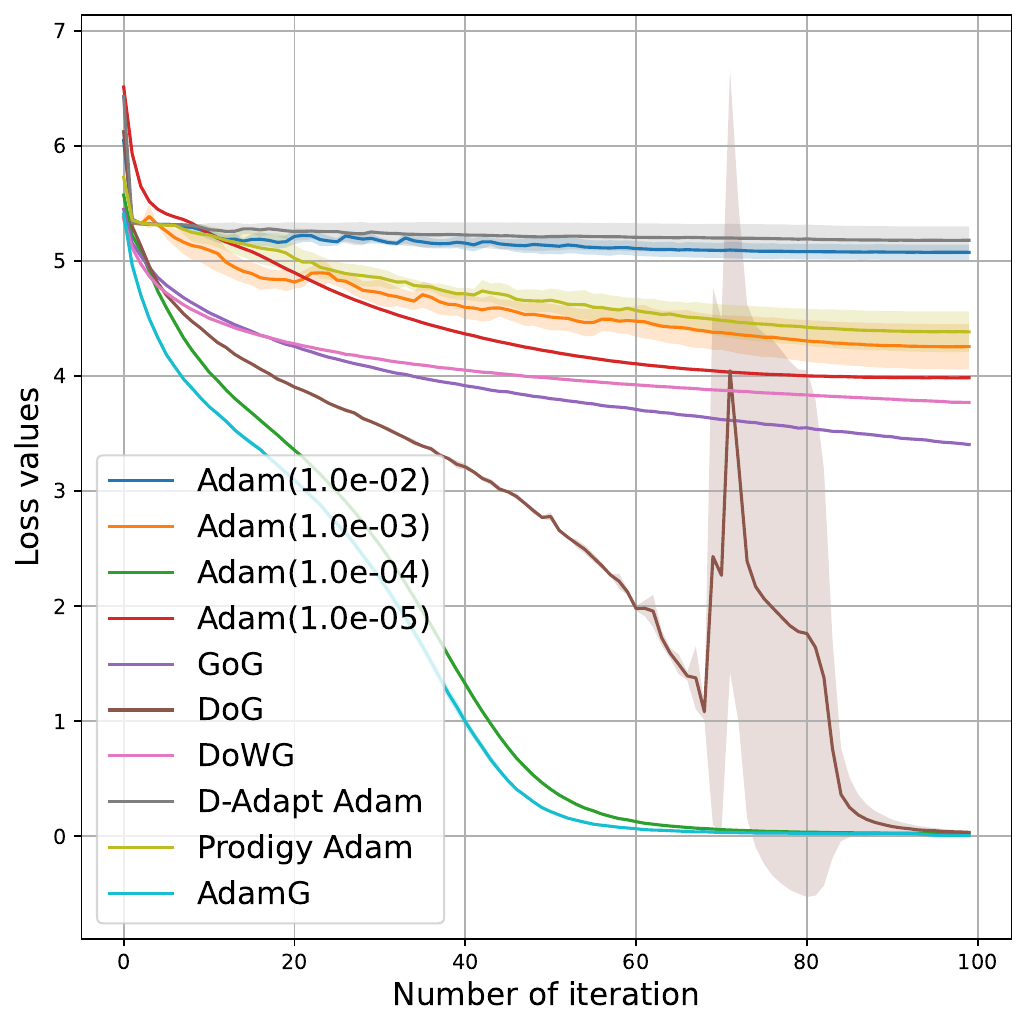}}
\subfigure[VGG11\&R.I.]{\includegraphics[width=0.24\linewidth]{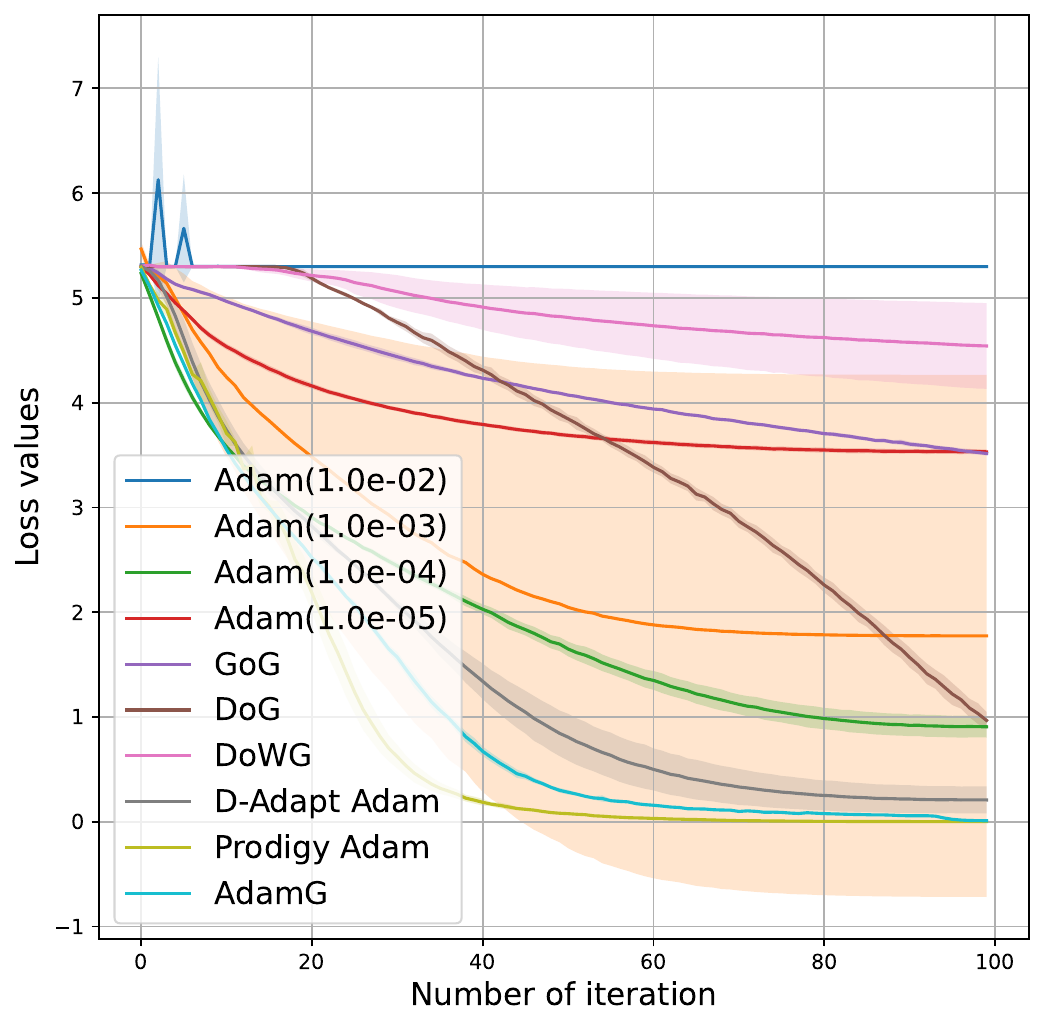}}
\end{center}
\caption{Tiny-ImageNet experiments. R.I. denotes randomly initialized networks.\label{app_fig_TINYIMAGENET}}
\end{figure*}


\begin{figure*}[!ht]
\begin{center}
\subfigure{\includegraphics[width=1.\linewidth]{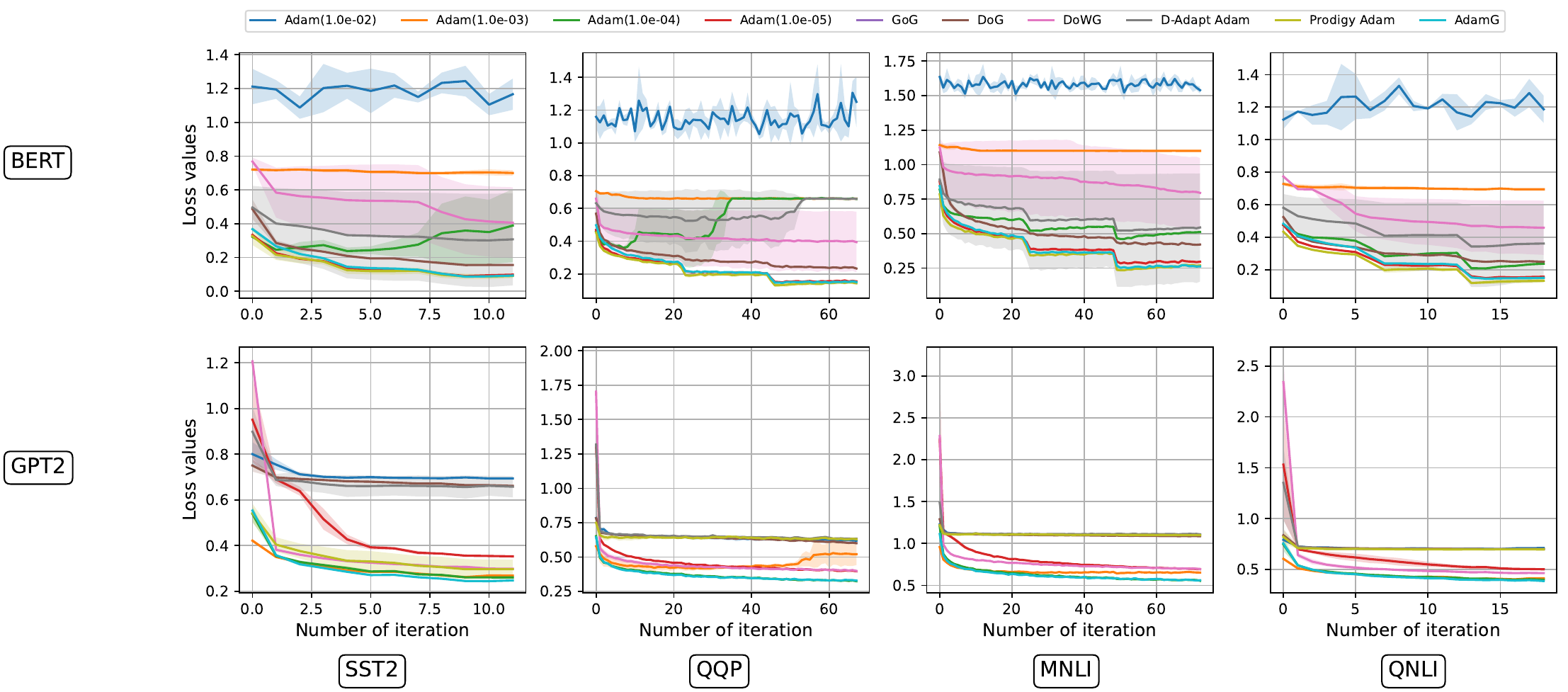}\label{lan_convergence_combined}}
\end{center}
\caption{BERT and GPT-2 under GLUE benchmark experiments.
\label{fig_lan_exp_comb}}
\end{figure*}

\end{document}